\documentclass[twoside,11pt]{article}

\usepackage[abbrvbib,nohyperref,preprint]{jmlr2e}
\usepackage[ref,caption]{leaf}

\usepackage{listings}
\usepackage{xcolor}

\definecolor{codegreen}{rgb}{0,0.6,0}
\definecolor{codegray}{rgb}{0.5,0.5,0.5}
\definecolor{codepurple}{rgb}{0.58,0,0.82}
\definecolor{backcolour}{rgb}{0.95,0.95,0.92}
\definecolor{dark-blue}{rgb}{0.15,0.15,0.4}
\definecolor{medium-blue}{rgb}{0,0,0.5}

\lstdefinestyle{mystyle}{
    backgroundcolor=\color{backcolour},   
    commentstyle=\color{codegreen},
    keywordstyle=\color{magenta},
    numberstyle=\tiny\color{codegray},
    stringstyle=\color{codepurple},
    basicstyle=\ttfamily\footnotesize,
    breakatwhitespace=false,         
    breaklines=true,                 
    captionpos=b,                    
    keepspaces=true,                 
    numbers=left,                    
    numbersep=5pt,                  
    showspaces=false,                
    showstringspaces=false,
    showtabs=false,                  
    tabsize=2
}

\hypersetup{
  colorlinks, linkcolor={dark-blue},
  citecolor={dark-blue}, urlcolor={medium-blue}
}

\newcommand{\params}{\theta}

\ShortHeadings{}{On Uncertainty, Tempering, and Data Augmentation in Bayesian Classification}

\begin{document}

\title{On Uncertainty, Tempering, and Data Augmentation in Bayesian Classification}

\author{\name Sanyam Kapoor\thanks{Equal contribution.}, \addr New York University  \\
      \name Wesley J. Maddox$^*$, \addr New York University \\
      \name Pavel Izmailov$^*$, \addr New York University \\
      \name Andrew Gordon Wilson, \addr New York University}

\maketitle

\begin{abstract}
Aleatoric uncertainty captures the inherent randomness of the data, such as measurement noise. In Bayesian regression, we often use a Gaussian observation model, where we control the level of aleatoric uncertainty with a noise variance parameter. By contrast, for Bayesian classification we use a categorical distribution with no mechanism to represent our beliefs about aleatoric uncertainty. Our work shows that explicitly accounting for aleatoric uncertainty significantly improves the performance of Bayesian neural networks. We note that many standard benchmarks, such as CIFAR, have essentially no aleatoric uncertainty. Moreover, we show data augmentation in approximate inference has the effect of softening the likelihood, leading to underconfidence and profoundly misrepresenting our honest beliefs about aleatoric uncertainty. Accordingly, we find that a cold posterior, tempered by a power greater than one, often more honestly reflects our beliefs about aleatoric uncertainty than no tempering --- providing an explicit link between data augmentation and cold posteriors. We show that we can match or exceed the performance of posterior tempering by using a Dirichlet observation model, where we explicitly control the level of aleatoric uncertainty, without any need for tempering.
\end{abstract}

\section{Introduction}
\label{sec:intro}

Uncertainty is often compartmentalized into \emph{epistemic uncertainty} and \emph{aleatoric uncertainty} \citep{Hora1996AleatoryAE,Kendall2017WhatUD,Malinin2018PredictiveUE}. Epistemic uncertainty, sometimes called \emph{model uncertainty}, is the reducible uncertainty over which solution is correct given limited information. Bayesian methods naturally represent epistemic uncertainty through a distribution over model parameters, leading to a posterior distribution over functions that are consistent with data. As we observe more data, this posterior distribution concentrates around a single solution. Aleatoric uncertainty is intrinsic irreducible uncertainty, often representing measurement noise in regression, or mislabeled training points in classification \citep[e.g.][]{beyer2020we}. Although measurement noise can be reduced, for example, with better instrumentation, it is often a fixed property of the data we are given. Correctly expressing our assumptions about aleatoric uncertainty is crucial for achieving good predictive performance, both with Bayesian and non-Bayesian models 
\citep{Senge2014ReliableCL}.

In particular, our assumptions about aleatoric uncertainty profoundly affect predictive \emph{accuracy}, not only predictive uncertainty.\footnote{Epistemic uncertainty also has a significant effect on predictive accuracy, as discussed in \citet{wilson2020bayesian}.} In \cref{fig:conceptual}(a,b) we show the predictive distributions of two Gaussian process regression models \citep{rasmussen2006gaussian} trained on the same data and using the same RBF kernel function.
The only difference between these models is their assumptions about the aleatoric uncertainty.
The model in \cref{fig:conceptual}(a) assumes a high observation noise (each point is corrupted by $\mathcal{N}(0,\sigma^2 )$, with $\sigma^2 = 1$).
Consequently, this model explains many of the observations with noise --- such that the predictive mean does not closely fit the training datapoints.
The model in \cref{fig:conceptual}(b) assumes a low observation noise ($\sigma^2 = 10^{-2}$), and consequently the predictive mean runs through the training data, leading to very different predictions compared to the model in panel (a).

\begin{figure*}[!t]
\centering
    
\begin{tabular}{c cc cc}

\hspace{-0.4cm}\includegraphics[height=.277\linewidth]{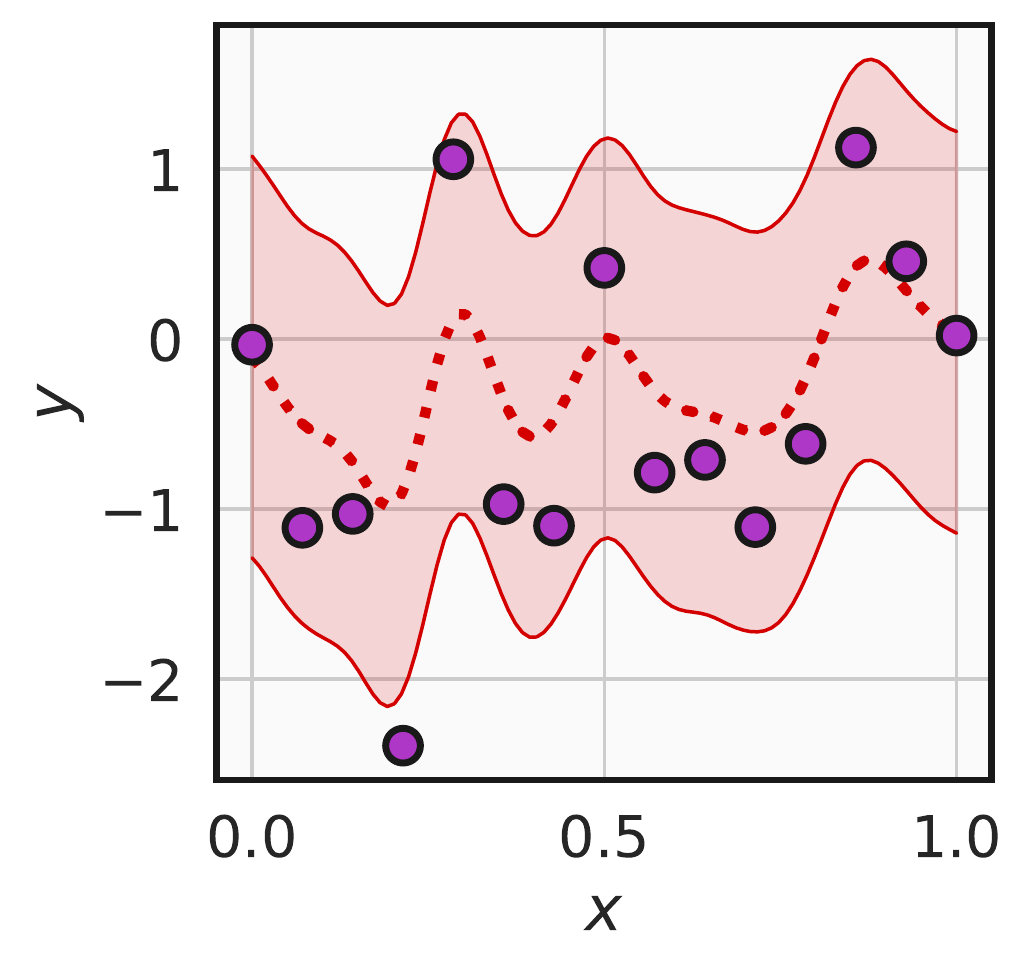}
&&
\hspace{-.7cm}\includegraphics[height=.277\linewidth]{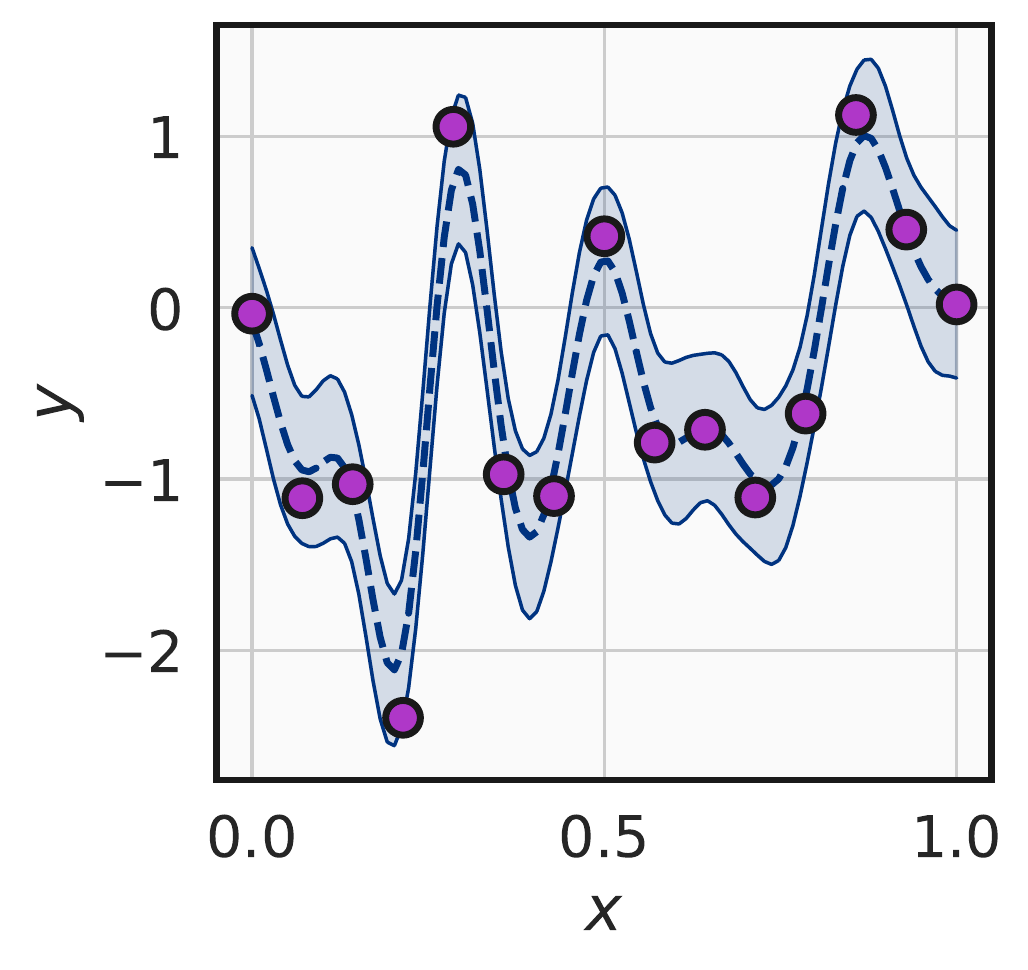}
&\quad\quad&
\includegraphics[height=.27\linewidth]{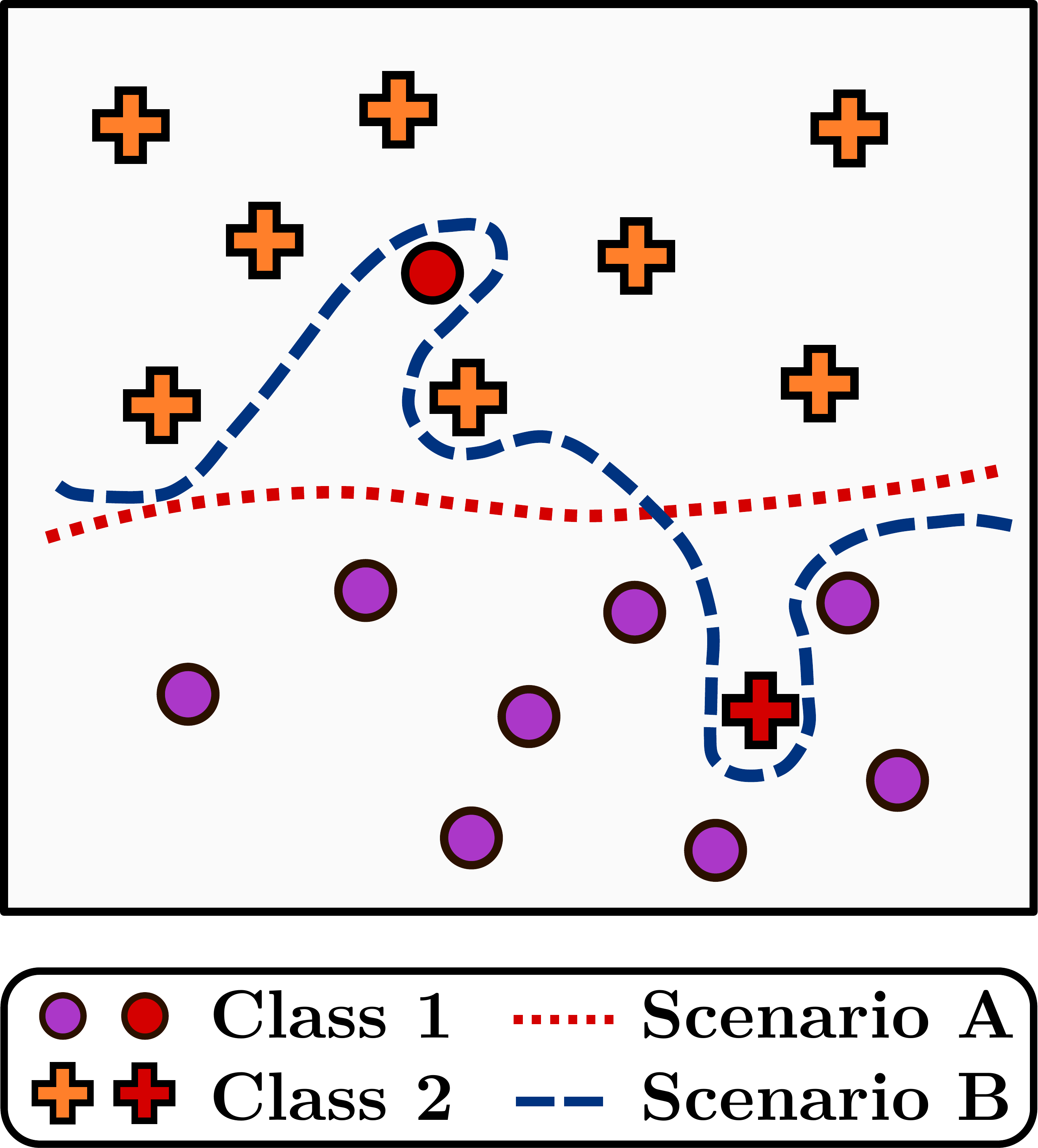}
\\[1mm]
(a) GP regression, $\sigma^2 = 1$ &&
\hspace{-.4cm}(b) GP regression, $\sigma^2 = 10^{-2}$&&
(c) Classification
\end{tabular}
\caption{
    \textbf{Effect of aleatoric uncertainty.}
    In regression problems, we can express our assumptions about aleatoric uncertainty by setting the observation noise parameter $\sigma^2$.
    \textbf{(a)}: Gaussian process regression with high observation noise $\sigma^2$ explains many of the observations with noise.
    Here the datapoints are shown with purple circles, the dashed line shows the predictive mean, and the shaded region shows one standard deviation of the predictive distribution. 
    \textbf{(b)}: Gaussian process regression on the same data, but with low observation noise $\sigma^2$ fits the training data nearly perfectly.
    \textbf{(c)}: In classification problems, we do not have a direct way to specify our assumptions about aleatoric uncertainty.
    In particular, we might use the same Bayesian neural network model if we know the data contains label noise (scenario A) and if we know that there is no label noise (scenario B), leading to poor performance in at least one of these scenarios.
}
\label{fig:conceptual}
\end{figure*}

Our assumptions about the aleatoric uncertainty are just as important in classification.
We illustrate this point in \cref{fig:conceptual}(c), where we fix a dataset and consider two possible scenarios.
\textbf{Scenario A}: suppose we know that the data contains label noise, such that some of the training examples are labeled incorrectly.
In this case, it is reasonable to believe that the data points shown in red in \cref{fig:conceptual}(c) are highly likely to be incorrectly labeled, and we should opt for a solution with a decision boundary shown with a dotted red line.
\textbf{Scenario B}: now, suppose that we know that the data are correctly labeled.
In this case, we want our model to perfectly classify all the training data and would prefer a solution shown by the dashed blue line.
Both of these solutions are reasonable descriptions of the data, and we can only make an informed choice between them by incorporating our assumptions about the aleatoric uncertainty.

While in regression problems we can easily express our assumptions about aleatoric uncertainty, such as through a noise variance parameter in a Gaussian likelihood, 
in classification we do not have a direct way to express these beliefs.
Indeed, in practice we might use the same Bayesian neural network model (with the same likelihood and prior) in both scenarios A and B above, necessarily leading to poor performance in at least one of these scenarios.

In this paper, we investigate ways of expressing our beliefs about aleatoric uncertainty in classification. In \cref{sec:aleatoric_uncertainty_understanding} we show that the standard softmax likelihood is equivalent to a Dirichlet likelihood over the predicted class probabilities, providing a mechanism to directly reason about aleatoric uncertainty in classification. Using this interpretation, we show that posterior tempering is a natural way of expressing beliefs about aleatoric uncertainty: a ``cold posterior'', raised to a power $1 / T$ with $T < 1$, corresponds to increasing the concentration parameter of the Dirichlet distribution for the predicted probability of the observed class, forcing the model to be confident on the train data. As an alternative, we consider a \emph{noisy Dirichlet model}, which can be viewed as an input dependent prior, to provide direct control over aleatoric uncertainty for Bayesian neural network classifiers. This approach achieves competitive performance on image classification problems with a valid likelihood and no need for tempering. In \cref{sec:data_aug_effects}, we show theoretically how data augmentation counterintuitively softens the likelihood, leading to solutions that underfit the training data. These results precisely characterize the empirical link between data augmentation and the success of cold posteriors observed by \citet{izmailov2021bayesian} and \citet{fortuin2021bayesian}, and show that tempering in fact serves to \emph{correct} for the effects of data augmentation on representing aleatoric uncertainty. In \cref{sec:experiments} we exemplify several of the conceptual findings in previous sections.

In short, we show how it is possible to explicitly characterize aleatoric uncertainty in Bayesian neural network classification, and how we can naturally accommodate data augmentation in a Bayesian setting. We find that cold posteriors in Bayesian neural networks more honestly reflect our beliefs about aleatoric uncertainty than $T=1$, as data augmentation softens the likelihood, and many standard benchmarks like CIFAR have essentially no aleatoric uncertainty for the observed training points. While \citet{wenzel2020good} note that cold posteriors can provide good performance in Bayesian neural network classifiers, we show that cold posteriors are not required --- simple modifications to the likelihood to reflect our honest beliefs about aleatoric uncertainty provide comparable performance.

\section{Related Work}
\label{sec:related}

Early work on Bayesian neural networks (BNNs) focused on hyperparameter learning and mitigating overfitting, using Laplace approximations, variational methods, and Hamiltonian Monte Carlo based MCMC \citep[e.g.][]{Mackay1992APB, hinton1993keeping, Neal1995BayesianLF}. More recently, \citet{wilson2020bayesian} show how Bayesian model averaging is especially compelling for \emph{modern} deep networks --- covering deep ensembles \citep{lakshminarayanan2017simple} as approximate Bayesian inference, induced priors in function space, mitigating double descent, generalization behaviour in deep learning \citep{zhang2021understanding}, posterior tempering \citep{grunwald2012safe, wenzel2020good}, and connections with loss surface structure such as mode connectivity \citep{garipov2018loss}. A number of recent works have focused on making Bayesian deep learning practical \citep[e.g.,][]{wilson2016deep, maddox2019simple, osawa2019practical, zhang2019cyclical, wilson2020bayesian, dusenberry2020efficient, daxberger2021laplace}, often with better results than classical training, and essentially no additional runtime overhead. BNNs also have an exciting range of applications, from astrophysics \citep{cranmer2021bayesian}, to click-through rate prediction \citep{liu2017pbodl}, to diagnosis of diabetic retinopathy \citep{filos2019systematic}, to fluid dynamics \citep{geneva2020modeling}.

\citet{wenzel2020good} demonstrated with several examples that raising the posterior to a power $1/T$, with $T<1$, in conjunction with SGLD inference, improves performance over classical training, but $T=1$ can provide worse results than classical training, in what has been termed the \emph{cold posterior effect}. However, in a study using HMC inference \citet{izmailov2021bayesian} showed that for \emph{all} cases considered in \citet{wenzel2020good} there is no cold posterior effect when data augmentation is removed. Indeed, while \citet{Noci2021DisentanglingTR} claims that data augmentation only explains the cold posterior effect for CIFAR-10 but not IMDB in \citet{wenzel2020good}, \citet{izmailov2021bayesian} in fact also shows no cold posterior effect on IMDB without data augmentation. The sufficiency of data augmentation to observe the cold posterior effect has since been confirmed by several works using SG-MCMC inference \citep{fortuin2021bayesian,Noci2021DisentanglingTR,nabarro2021data}. 

Several works have suggested that misspecified priors explain cold posteriors, arguing that the current default choice of isotropic Gaussian priors in BNNs are inadequate \citep{wenzel2020good,zeno2020cold,fortuin2021bayesian}. 
However, \citet{fortuin2021bayesian} find that the cold posterior effect cannot be alleviated for convolutional neural networks by using heavy tailed or correlated priors under data augmentation.
While \citet{fortuin2021bayesian} claim a cold posterior effect on Fashion MNIST for fully connected 
MLPs with Gaussian priors and no data augmentation, their experiments show very minimal effects (about $0.25\%$ on test accuracy on Fashion MNIST) that are generally not present 
in terms of calibration or out-of-distribution detection. Moreover, \citet{wilson2020bayesian} show that the experiments in \citet{wenzel2020good} suggesting a poor prior are easily resolved
by tuning the variance scale of the prior, and that isotropic Gaussian priors are practically effective and provide desirable properties in function space. \citet{izmailov2021bayesian} also show that 
standard Gaussian priors perform similarly to a range of other priors, such as heavy-tailed logistic priors, and mixture of Gaussian priors, and generally are high performing, providing better results
than standard training and deep ensembles with HMC inference.

Separately from cold posteriors, \citet{izmailov2021dangers} explain how standard priors for Bayesian neural networks can cause profound deteriorations in performance under covariate shift --- 
with the potential to impact virtually any real-world application of BNNs. They introduce the input dependent \emph{EmpCov} priors, which helps remedy this issue.

\citet{Noci2021DisentanglingTR} argue that many different factors --- likelihoods, priors, data augmentation --- can cause cold posteriors, and that there may be no one cause in general. 
These observations are actually aligned with the earlier work by \citet{wilson2020bayesian}, which argues that tempering can partially address a wide range of misspecifications, such that it
would be surprising if $T=1$ in general provided the best performance. In other words, if our model is at all misspecified we would expect $T=1$ to be suboptimal, and since in any realistic
setting our model will not be perfectly specified, despite our best efforts, it is unreasonable to demand that $T=1$ or be particularly alarmed if it is not.
\citet{Noci2021DisentanglingTR} further suggest that Gaussian priors may be putting weight on complicated hypotheses, which ends up hurting performance, while \citet{wilson2020bayesian} on the other hand demonstrate how Gaussian priors over parameters induce useful priors in function space. Either way, as above, 
the cold posterior effect is nearly removed in practice if we remove data augmentation, including all examples in \citet{wenzel2020good}. 

\citet{aitchison2020statistical} argue that BNN likelihoods are misspecified since many benchmark datasets have been carefully curated by human labelers; when we rely on a curation strategy that discards any samples which do not arrive at label consensus, our likelihood has the form ${p(y\mid x)^H}$ for $H$ human labelers, which connects to likelihood tempering. They also show that posterior tempering is less helpful in the presence of label noise. \citet{adlam2020cold} also consider model misspecification, showing a cold posterior effect can arise even with exact inference in Gaussian process regression when the aleatoric uncertainty is mis-estimated. \citet{nabarro2021data} modify the likelihood to accommodate data augmentation, using as the observation model the average of the neural network outputs over augmentations, but still find a cold posterior effect --- despite the lack of a cold posterior without augmentation. 

Our paper makes several distinctive contributions in the context of prior work: (1) we argue that standard likelihoods do not represent our beliefs about aleatoric uncertainty, and that standard benchmarks have essentially no aleatoric uncertainty; (2) we show how tempering and data augmentation specify aleatoric uncertainty, beyond curation; (3) we show the \emph{precise way in which data augmentation with SGLD leads to underconfidence in the likelihood}, a counterintuitive result which finally resolves the empirical connection between data augmentation and cold posteriors; (4) we show that a $T<1$, particularly with data augmentation, is a more honest reflection of our beliefs about aleatoric uncertainty than $T=1$; (5) we show how a lognormal approximation of the Dirichlet likelihood, originally used for tractable Gaussian process classification \citep{Milios2018DirichletbasedGP}, can naturally reflect our beliefs about aleatoric uncertainty, and for the first time remove the cold posterior effect in the presence of data augmentation; (6) we show that priors can also be used to specify our beliefs about aleatoric uncertainty. The code to reproduce experiments is available at \url{https://github.com/activatedgeek/bayesian-classification}.

\section{Background}
\label{sec:background}

\textbf{Bayesian Model Averaging.}\quad With Bayesian inference, we aim to infer the
posterior distribution over parameters having observed a 
dataset ${\dset = \{(x_i,y_i)\}_{i=1}^N}$ of input-output pairs, given by
${p(\params\mid \dset) \propto p(\dset \mid \params)p(\params)}$
for a given observation likelihood under the i.i.d. assumption ${p(\dset\mid\params) = \prod_{i=1}^N p(y_i\mid x_i, \params)}$, and prior over parameters $p(\params)$. For any novel input $x_\star$, we estimate the posterior predictive distribution via \emph{Bayesian model averaging} (BMA) given by
\begin{align}
\begin{split}
p(y_\star \mid x_\star)	&= \int p(y_\star \mid x_\star,\params)p(\params\mid \dset) d\params . 
\end{split} \label{eq:bma}
\end{align}
This integral cannot be expressed in closed form for a Bayesian neural network, and is typically approximated with Variational Inference (VI), the Laplace
approximation, or Markov Chain Monte Carlo (MCMC) \citep[e.g.,][]{wilson2020bayesian, pml2Book}. 

\textbf{Cold Posteriors and Tempering.}\quad Let ${p(\dset \mid \params)}$ denote the likelihood function, $p(\params)$ 
the prior over parameters of the neural network, and $U(\params)$ the \emph{posterior energy function}. Following \citet{wenzel2020good}, we then define a \emph{cold 
posterior} for $T<1$ as
\begin{align}
p_{\mathrm{cold}}(\params \mid \dset) \propto \exp \bigg\{ -\frac{1}{T} \underbrace{ \left( -\log{p(\dset \mid \params)} - \log{p(\params)} \right)}_{U(\params)} \bigg\}, \label{eq:cold_posterior}
\end{align}
which is effectively raises both the likelihood and the prior to a power $1/T$. $T=1$ recovers the standard Bayes' posterior. 
By comparison, a \emph{tempered likelihood posterior} (e.g. \citet{grunwald2012safe}) only raises the likelihood term to a power $1/T$ as
\begin{align}
p_{\mathrm{temp}}(\params \mid \dset) \propto \exp\bigg\{- \left(-\frac{1}{T}\log{p(\dset \mid \params)} - \log{p(\params)}\right)\bigg\}. \label{eq:tempered_posterior}
\end{align}

\textbf{Stochastic Gradient Langevin Dynamics (SGLD).}\quad For large-scale neural 
networks, exact posterior inference remains intractable. To approximate samples from the posterior ${p(\params \mid \dset)}$ over parameters of a neural network, we simulate the time-discretized Langevin stochastic differential equation (SDE) \citep{Srkk2019AppliedSD,welling2011bayesian}
\begin{align}
\begin{split}
\Delta \params &= \mbf{M}^{-1} \mbf{m} \epsilon, \\
\Delta \mbf{m} &= - \nabla_\params \widetilde{U}(\params) \epsilon - \gamma \mbf{m} \epsilon + \sqrt{2\gamma T} \eta, \textrm{ where } \eta \sim \gaussian{\mbf{0}, \mbf{M}},
\end{split} \label{eq:langevin_sde}
\end{align}
where $\mbf{m}$ are the auxiliary momentum variables, $\mbf{M}$ is
the mass matrix which acts as a preconditioner (often identity), $\gamma$ is the friction parameter, $T$ is the temperature, 
$\Delta t = \epsilon$ is the time discretization (step size),
and $\nabla_\params \widetilde{U}(\params)$ is an unbiased estimator of the gradient $\nabla_\params U(\params)$ using only a subset of the dataset $\dset$ for computational efficiency. 
For $\epsilon \to 0$ in the limit of time $t \to \infty$, simulating \cref{eq:langevin_sde} produces a trajectory 
distributed according to the stationary distribution 
${\exp\left\{ -U(\params)/T \right\}}$, which is exactly the posterior
$p_{\mathrm{cold}}$ in \cref{eq:cold_posterior} we desire.
When $\gamma = 0$, \cref{eq:langevin_sde} describes the Stochastic Gradient Langevin Dynamics (SGLD) \citep{welling2011bayesian}, and otherwise it describes the Stochastic Gradient Hamiltonian Monte Carlo (SGHMC) where $1-\gamma$ represents the momentum \citep{cheni14}. 
Furthermore, we can sample from \cref{eq:tempered_posterior} by setting $T = 1$ in \cref{eq:langevin_sde} and raising only the likelihood to a power $T$. In addition, it is often beneficial to use a cyclical time-stepping schedule for $\epsilon$ in \cref{eq:langevin_sde}, as proposed by \citet{zhang2019cyclical} for cyclical-SGLD (cSGLD) when $\gamma = 0$, and cyclical-SGHMC (cSGHMC) when $\gamma > 0$.

\textbf{Bayesian Classification.}\quad For a $C$-class classification problem, a standard choice of observation likelihood is the categorical distribution ${p(y \mid x,\params) = \mathrm{Cat}([\pi_1,\pi_2,\dots,\pi_C])}$, where each class probability is computed using a softmax transform of logits ${f(x;\params) \in \reals^C}$ as $\pi_c \propto \exp\{ [f(x; \params)]_c \}$, and hence called the \emph{softmax likelihood}. The negative log of the softmax likelihood, i.e. $-\log{p(y \mid x,\params)}$ is exactly the standard cross-entropy loss. The prior over parameters $p(\params)$ is often chosen to be an isotropic Gaussian $\gaussian{0, \sigma^2 \mbf{I}}$ with some fixed variance $\sigma^2$. Approximate posterior samples are often obtained with SGLD, and then used to approximate the Bayesian model average in \cref{eq:bma} using simple Monte Carlo.

\section{Aleatoric Uncertainty in Bayesian Classification}
\label{sec:aleatoric_uncertainty_understanding}

We will now discuss how we represent aleatoric uncertainty in classification (\cref{sec: howaleatoric}), how tempering reduces aleatoric uncertainty (\cref{sec: liktemp}), and how
we can explicitly represent aleatoric uncertainty without tempering in a modified Dirichlet observation model (\cref{sec: ndm}). In \cref{sec:data_aug_effects} we will use these 
foundations to show precisely how data augmentation affects our representation of aleatoric uncertainty.

\subsection{How do we represent aleatoric uncertainty in classification?}
\label{sec: howaleatoric}

Let us consider the Bayesian neural network posterior over parameters $w$ in a classification problem:
\begin{equation}
    \label{eq:posterior}
    p(w \mid D)
    \propto
    p(w) \prod_{x, y \in \mathcal D} f_y(x, w),
\end{equation}
where we denote the output of the softmax layer of the model corresponding to class $y$ on input $x$ as $f_y(x, w)$.
We can think of the class probability vectors $f(x) = (f_1(x, w), \ldots, f_C(x, w))$ as latent variables, where the prior distribution $p(w)$ over
the parameters of the network implies a joint prior distribution over $\{f(x)\}_{x \in \mathcal D}$.
We will initially focus on the observation model.
First, we note that we can introduce a uniform prior over the predicted class probabilities $f(x)$:
\begin{equation}
    \label{eq:adding_uniform_prior}
    p(y \mid f(x)) = f_y(x) \propto
    \text{Dir.}(1, \ldots, 1)(f(x)) \cdot f_y(x),
\end{equation}
where Dir. denotes the Dirichlet distribution and $\text{Dir.}(1, \ldots, 1)$ is a uniform distribution over the class probabilities $f(x)$ (with support on the unit $C$-simplex), and the proportionality is with respect to the latent variables $f(x)$.
We can view the right hand side of Eq. \eqref{eq:adding_uniform_prior} as the unnormalized posterior in a multinomial model, where we use a uniform prior over the the parameters $f(x)$, and observe a single count of class $y$.
Using the fact that the Dirichlet distribution is conjugate to the multinomial likelihood \citep[e.g.,][Ch. 2.2.1]{bishop06}, we can rewrite the right hand side of \cref{eq:adding_uniform_prior} as follows:
\begin{equation}
    \label{eq:likelihood_dirichlet}
    p(y \mid f(x)) \propto
    \text{Dir.}(1, \ldots, 1, \underbrace{2}_{\text{position}~y}, 1, \ldots, 1)(f(x)).
\end{equation}
Eq. \eqref{eq:likelihood_dirichlet} describes the distribution induced by the observation $y$ on the predicted class probabilities for the corresponding input $x$.
The posterior over the parameters $w$ of the network can then be written as a product of Eq. \eqref{eq:likelihood_dirichlet} and the prior $p(w)$:
\begin{equation}
    \label{eq:posterior_dir}
    p(w \mid D)
    \propto
    p(w) \prod_{x, y \in \mathcal D} \text{Dir.}(1, \ldots, 1, \underbrace{2}_{\text{position}~y}, 1, \ldots, 1)(f(x)).
\end{equation}
\cref{eq:posterior_dir} provides intuition for how Bayesian neural networks estimate aleatoric uncertainty in classification.
If we ignore the prior $p(w)$ and assume that the implied prior over $f(x)$ is uniform, then for an observation $(x, y)$ the posterior over $f(x)$ is Dir.$(1, \ldots, 2, \ldots, 1)$ and the posterior mean for the probability of the correct class $\mathbb E f_y(x)$ is $\frac 2 {C + 1}$.
For example, for a dataset with $100$ classes (e.g. CIFAR-100), the model on average will only be $2 / 101 \approx 2\%$ confident in the correct label on the \textit{training data}!

In practice, the prior $p(w)$ will imply a non-trivial joint prior distribution over the latent variables $\{f(x)\}_{x \in \mathcal D}$, so the actual posterior may be more (or less) confident than suggested by the analysis above.
Furthermore, we have very limited understanding of the implied distribution \citep[see][for empirical analysis]{wenzel2020good, wilson2020bayesian}.
In particular, most practitioners use simple $p(w) = \mathcal N(0, \alpha^2 I)$ priors, regardless of the amount of label noise in the data.
Explicitly constructing a prior in the parameter space that would lead to highly confident $f(x)$ is challenging, but we will show how modifying the likelihood 
can be viewed as providing an input-dependent prior that more fully reflects our beliefs about aleatoric uncertainty.

The expression in \cref{eq:posterior_dir} suggests two natural ways of modifying the posterior to account for the aleatoric uncertainty:
increasing the Dirichet concentration for the observed class (equivalent to posterior tempering) and decreasing the concentration for the unobserved classes (noisy Dirichlet model).
Below, we describe both of these approaches in detail.

\subsection{Likelihood tempering reduces aleatoric uncertainty}
\label{sec: liktemp}

The tempered likelihood posterior (see \cref{sec:background}) corresponding to the posterior in \cref{eq:posterior} in BNN classification can be written as
\begin{equation}
    \label{eq:temp_posterior}
    p_{temp}(w \mid D)
    \propto
    p(w) \prod_{x, y \in \mathcal D} f_y(x, w)^{1/T},
\end{equation}
where $T$ is the temperature.
Analogously to the derivation above, we can rewrite this posterior as 
\begin{equation}
    \label{eq:temp_posterior_dirichlet}
    p_{temp}(w \mid D)
    \propto
    p(w) \prod_{x, y \in \mathcal D} \text{Dir.}\bigg(1, \ldots, 1, \underbrace{1 + \frac 1 T}_{\text{position}~y}, 1, \ldots, 1\bigg)(f(x)).
\end{equation}
In other words, the tempered posterior corresponds to the same model as the regular posterior, but assumes that we observed $1 / T$ counts of class $y$ for each input $x$ in the training data.
In particular, assuming the prior $p(w)$ implies a uniform distribution over $f(x)$, the confidence in the correct labels on the train data under the tempered posterior is $\mathbb E f_y(x)$ is $\frac {1 + 1 / T} {C + 1 / T} = \frac{T + 1}{C T + 1}$.
For a dataset with $100$ classes and at temperature $T = 10^{-2}$ we get an average confidence of $50.5\%$, much higher than the $2\%$ for the standard observation model.

\textbf{Tempering the Likelihood vs Tempering the Posterior.}\quad
Prior work has mostly considered tempering the full Bayesian posterior in \cref{eq:cold_posterior} as opposed to just tempering the likelihood in \cref{eq:tempered_posterior}. 
In \cref{sec:cp_and_sharpness} we show that tempering the full Bayesian posterior is almost always equivalent to changing the prior distribution, and tempering the likelihood.
Moreover, we show that tempering just the likelihood recovers the same cold posterior effect \citep{wenzel2020good} as tempering the full posterior.

\textbf{How should we think about tempering?}\quad
Prior work has asserted that posterior tempering sharply deviates from the Bayesian paradigm and the cold posterior effect is highly problematic \citep{wenzel2020good,Noci2021DisentanglingTR,fortuin2021bayesian}.
We, on the other hand, argue that likelihood tempering is in fact a practical way to incorporate our assumptions about aleatoric uncertainty.
Relatedly, \citet{aitchison2020statistical} argue that the cold posterior effect can be caused by data curation.
For example, CIFAR-$10$ and ImageNet datasets have very relatively little label noise and are carefully curated \citep[e.g.][]{pleiss2020identifying}.
Thus, we may
expect that tempered likelihoods with low temperatures will lead to optimal performance.
\citet{wilson2020bayesian} also argue that posterior tempering can be viewed as a change of the observation model.

\textbf{Is Tempered Likelihood a Valid Likelihood?}\quad
In classical Bayesian inference, the observation model is a distribution $p(y \mid x)$ over the labels conditioned on the input.
\citet{wenzel2020good} argue that the tempered softmax likelihood is in general not a valid likelihood because 
it does not sum to $1$ over classes.
However, \citet{wenzel2020good} show, the tempered softmax likelihood with $T < 1$ can be interpreted as a valid likelihood if we introduce a new class, which is not observed in the training data. While
they discard this interpretation as incoherent, it is not necessarily unreasonable to include an unobserved class in the model.
Moreover, from \cref{eq:temp_posterior_dirichlet}, we can naturally interpret the tempered likelihood as using the multinomial observation model, assuming $1 / T$ counts of the label are observed for each of the training datapoints, which is uncontroversial and perfectly valid.

\subsection{Noisy Dirichlet model: changing the prior over class probabilities}
\label{sec: ndm}

As we have seen in \cref{eq:temp_posterior_dirichlet}, likelihood tempering increases the posterior confidence by increasing the Dirichlet distribution concentration parameter for the observed class $y$ from $2$ to $1 + 1/T$. 
We can achieve a similar effect by \textit{decreasing the concentration of the unobserved classes} instead\footnote{
While for concreteness we discuss increasing the confidence of the model here, we can equivalently decrease the confidence of the model by \textit{increasing} the concentration parameters for the unobserved classes, if we expect the level of label noise to be high.}!
Indeed, consider the distribution
\begin{equation}
    \label{eq:posterior_dir_noisy}
    p_{ND}(w \mid D)
    \propto
    p(w) \prod_{x, y \in \mathcal D} \text{Dir.}(\alpha_{\epsilon}, \ldots, \alpha_{\epsilon}, \underbrace{\alpha_{\epsilon}+1}_{\text{position}~y}, \alpha_{\epsilon}, \ldots, \alpha_{\epsilon})(f(x)),
\end{equation}
where ND in $p_{ND}$ stands for \textit{Noisy Dirichlet} and $\alpha_{\epsilon}$ is a tunable parameter.
The noisy Dirichlet model was originally proposed by \citet{Milios2018DirichletbasedGP} in the context of Gaussian process classification,
where they designed a tractable approximation to this model.
Using our running example, if $p(w)$ induces a uniform distribution over $f(x)$, for a problem with $100$ classes and $\alpha_{\epsilon} = 10^{-2}$ we have expected confidence 
$\mathbb E f_y(x) = \frac {\alpha_\epsilon + 1} {C\alpha_\epsilon + 1} = 50.5\%$, which is the same as for the tempered likelihood with $T=10^{-2}$.

Now, again using the conjugacy of the Dirichlet and multinomial distributions, we can rewrite $p_{ND}$ as follows:
\begin{equation}
    \label{eq:posterior_dir_noisy_final}
    p_{ND}(w \mid D)
    \propto
    \overbrace{p(w) \prod_{x, y \in \mathcal D} \text{Dir.}(\alpha_{\epsilon}, \ldots, \alpha_{\epsilon})(f(x))}^{q_{ND}(w)} \cdot f_y(x) = q_{ND}(w) f_y(w).
\end{equation}
We can interpret $q_{ND}(w) = p(w) \cdot \prod_{x \in \mathcal D} \text{Dir.}(\alpha_{\epsilon}, \ldots, \alpha_{\epsilon})(f(x))$ as a prior over the parameters $w$ of the model.
Indeed, $q_{ND}(w)$ does not depend on the labels $y$, and simply forces the predicted class probabilities to be confident in any one of the classes for the training data.
See the \textit{EmpCov} prior in \citet{izmailov2021dangers} for another example of a prior that depends on the training inputs but not training labels.

The Noisy Dirichlet prior is intuitively appealing: in many practical settings a priori we believe that the aleatoric uncertainty on the training data is low, and the model should be confident in one of the classes.
At the same time, we would not want to enforce low aleatoric uncertainty everywhere in the input space.
Indeed, while we believe that the aleatoric uncertainty should be low on the training data, we do not expect that all the possible inputs to the model should be classified with high confidence.

In short, the Noisy Dirichlet model corresponds to using a \textit{valid likelihood} --- the standard softmax likelihood --- with a prior that explicitly enforces the model predictions to be confident on the training data.
In \cref{sec:exp_bnn_images} we will see that the noisy Dirichlet model removes the cold posterior effect: tempering is not needed to achieve strong performance.

\textbf{Gaussian Approximation.}
\quad
\citet{Milios2018DirichletbasedGP} considered the distribution in \cref{eq:posterior_dir_noisy_final} in the context of Gaussian process classification,
but with a different goal:
they aimed to create a regression likelihood which would approximate the softmax likelihood.
They further approximated the Dirichlet distribution $\text{Dir.}(f(x))$ over the class probabilities $f(x)$ with a product of independent Gaussian distributions over the logits $z(x)$:
\begin{equation}
\begin{split}
    \label{eq:posterior_dir_noisy_gaussian}
    p_{NDG}(w \mid D)
    \propto
    p(w) \prod_{x, y \in \mathcal D} \prod_{c=1}^C \mathcal N (z_c(x) \mid \mu_c, \sigma_c^2),~~\text{with} \\
    \alpha_c = 1 + \alpha_{\epsilon} \cdot I[c = y],~~
    \sigma_c^2 = \log(1 / \alpha_c + 1),~~
    \mu_c = \log(\alpha_c) - \frac{\sigma^2_c}{2},
\end{split}
\end{equation}
where $I[c = y]$ is the indicator function equal to $1$ for $c = y$ and $0$ for $c \ne y$.
Here \emph{NDG} stands for \emph{Noisy Dirichlet Gaussian} approximation.
While theoretically we do not need to use the Gaussian approximation in \cref{eq:posterior_dir_noisy_gaussian} for Bayesian neural networks and can directly use \cref{eq:posterior_dir_noisy_final}, we found that in practice the approximation is much more stable numerically.
Indeed, the approximation in \cref{eq:posterior_dir_noisy_gaussian} amounts to solving a regression problem in the space of logits $z(x)$ with a Gaussian observation model.
In the experiments in this paper we will use the $p_{NDG}$ model.

\begin{figure}[!t]
    \centering
    \includegraphics[width=0.8\linewidth]{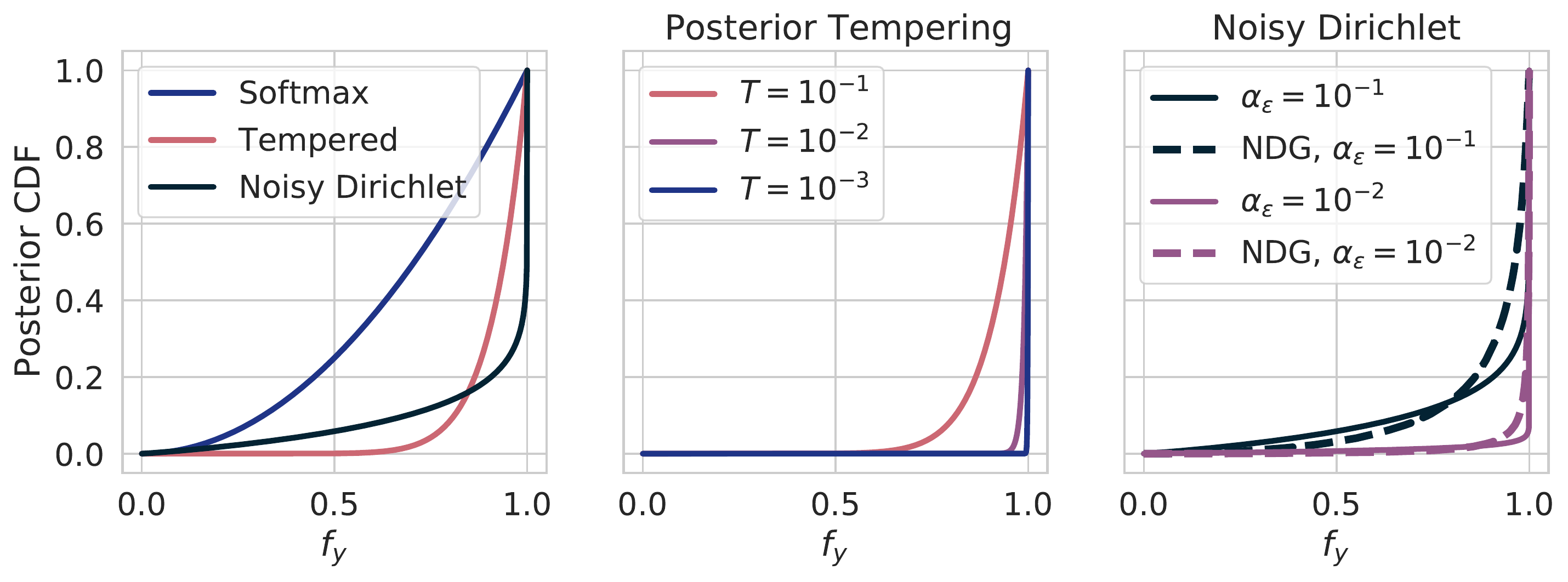}
    \caption{
    \textbf{Comparison of Bayesian classification models.}
    The probability CDF of the posterior over the confidence $f_y$ for the correct class $y$ with the standard cross entropy likelihood,
    tempered likelihood and the noisy Dirichlet model in a binary classification problem.
    \textbf{Left}: Tempering and noisy Dirichlet models both allow to concentrate the posterior on solutions that are confident in the correct label.
    \textbf{Middle}: Lower likelihood temperatures lead to higher concentration on confident solutions.
    \textbf{Right}: In the noisy Dirichlet model, lower values of the noise parameter $\alpha_{\epsilon}$ also lead to higher confidence.
    }
    \label{fig:coinflip}
\end{figure}

\textbf{Visual comparison.}\quad
In \cref{fig:coinflip} (left), we show the CDF of the posterior distribution over the confidence $f_y$ for the correct class $y$ with the standard softmax likelihood, tempered likelihood and the noisy Dirichlet model for a binary classification problem.
Here, we assume that the prior over $\{f(x)\}_{x \in \mathcal D}$ is uniform and independent across $x$ for the purpose of visualization.
Both tempering and the noisy Dirichlet model indeed push the predictions $f_y$ to be much more confident in the correct class $y$ compared to the standard softmax likelihood.
In \cref{fig:coinflip} (middle) we show the effect of the value of temperature: lower temperatures correspond to more confident predictions.
Similarly, in \cref{fig:coinflip} (right) lower values of $\alpha_\epsilon$ lead to more confident predictions in the noisy Dirichlet model.
The Gaussian approximation NDG matches the CDF of the noisy Dirichlet model well, but smooths it slightly, making it more amenable to numerical sampling.

\section{The Effect of Data Augmentation}
\label{sec:data_aug_effects}

Data augmentation is a key ingredient of modern deep learning.
Surprisingly, however, data augmentation has not been considered in detail in the context of Bayesian methods until recently.
Several works have shown that naive data augmentation can often cause the cold posterior effect in Bayesian deep learning \citep{izmailov2021bayesian,zeno2020cold,fortuin2021bayesian, nabarro2021data}.
We show that naive data augmentation, as typically applied in practice, is closely connected to likelihood tempering with a temperature $T > 1$.
In this case, the cold posteriors counterbalance the effect of data augmentation.

Data augmentation is a practical procedure, where at each iteration of optimization or sampling, we draw random augmentation for each object in each mini-batch of the data.
Consequently, we can only reason about the effect of the data augmentation on the limiting distributions of approximate inference procedures:
data augmentation is not defined for \textit{true} Bayesian neural networks.

Consider  the interaction between data augmenation and SGLD (or, equivalently, any other SGMCMC method). In SGLD, we aim to construct an unbiased estimator of the full gradient required in \cref{eq:langevin_sde}. 
In prior work, e.g. in \citet{wenzel2020good}, 
the stochastic gradient is estimated using a randomly augmented mini-batch of the data.
For a minibatch of
the full dataset $\dset_m = \{x_i,y_i\}_{i=1}^m \subset \dset$, and a finite set of augmentations $\mathcal{T} = \{ t_1,\dots,t_K \}$, 
this stochastic gradient is given by
\begin{align}
\begin{split}
\nabla \widetilde{U}(\params) =&~ \frac{N}{m} \sum_{(x_i,y_i) \in \dset_m} \nabla_\params \log{p(y_i \mid t_j(x_i))}
+ \nabla_\params \log{p(\params)},
\end{split}\label{eq:stochastic_gradient}
\end{align}
where the transformations $t_j$ are sampled uniformly from $\mathcal{T}$.

There are two sources of randomness in $\nabla \widetilde{U}(\params)$ --- the choice of the mini-batch $\dset_m$ and the choice of augmentation $t_j$ used for each $x_i$. The limiting distribution of SGLD is determined by the expectation of the stochastic gradient in \cref{eq:stochastic_gradient}, which is given by
\begin{align}
\begin{split}
\mathbb{E}[\nabla \widetilde{U}(\params)] &= \sum_{i=1}^N \mathbb{E}_{t_j \sim \mathcal{T}}[\nabla_\params \log{p(y_i \mid t_j(x_i))}] + \nabla_\params \log{p(\params)}, \\
&= \sum_{i=1}^N \sum_{j=1}^K [\nabla_\params \log{p(y_i \mid t_j(x_i))^{1/K}}] + \nabla_\params \log{p(\params)},
\label{eq:aug_sgld_grad_expectation}
\end{split}
\end{align}
where $t_j \sim \mathcal{T}$ are augmentations sampled uniformly from $\mathcal{T}$.
Therefore, we conclude that the limiting distribution of SGLD under data augmentation is given by 
\begin{align}
p_{\mathrm{aug}}(\params\mid \dset) \propto p(\params)\prod_{i=1}^N\prod_{j=1}^K p(y_i \mid t_j(x_i))^{1/K}.
\label{eq:implied_posterior}
\end{align}

We can interpret the limiting distribution in \cref{eq:implied_posterior} as a \emph{tempered likelihood posterior} for a new dataset $\dset^\prime = \{(t_j(x_i), y_i)\}_{i, j}$ which contains all augmentations of every data point from the original dataset $\dset$.
Furthermore, the likelihood is tempered with a temperature $K > 1$, corresponding to a \emph{warm posterior}.
In other words, the limiting distribution of SGLD with data augmentation corresponds to a posterior over the augmented dataset with a \emph{softened} likelihood, leading to less confidence about the labels.
By applying a cold temperature $T = 1 / K$ to the posterior in \cref{eq:implied_posterior}, we can recover the standard Bayesian posterior on the augmented dataset $\dset^\prime$.

In \cref{sec:exp_bnn_images}, we explore the effect of data augmentation on aleatoric uncertainty in practice and show that it does indeed soften the likelihood and lead to increased aleatoric uncertainty.

\textbf{What about Other Inference Procedures?}\quad
While the derivation in \cref{eq:implied_posterior} comes directly from studying the stochastic gradient evaluated in SGMCMC, we can equivalently derive the same posterior distribution in variational inference as we show in \cref{app:alt_inferences}. 

\textbf{Should we always use temperature $T=1/K$?}\quad
While using the temperature $T = 1 / K$ recovers the standard Bayesian posterior on the augmented dataset $\dset'$, it is not necessarily a correct approach to modeling data augmentation.
Indeed, the standard posterior  $p(\params \mid \dset') \propto p(\params)\prod_{i=1}^N\prod_{j=1}^K p(y_i \mid t_j(x_i))$ assumes indepdendence across both augmentations and data points.
In practice, however, we share the same label $y_i$ across all the augmented versions of the image; treating the observations as independent then leads to \textit{underestimating} the aleatoric uncertainty.
Indeed, consider the extreme scenario, where all the augmentations $t_j(x)$ simply return $x$.
In this case, treating the observations $(t_j(x), y)$ as independent will simply raise the likelihood to a power of $K$, while in reality we have not received any additional information from the augmentation policy.
Correcting the likelihood by a factor of $1/K$ is the correct approach only when the predictions on $t_j(x)$ are completely independent from each other (see \cref{sec:data_aug_gp}).
See also \citet{nabarro2021data} for a related discussion.

\textbf{A Proper Augmentation Likelihood.}\quad
As we have noted, simple tempering does not model the augmented data perfectly, as it ignores the dependencies between the observations. 
In \cref{sec:app_aug_lik}, we develop a proper likelihood that correctly accounts for data augmentation. 
In our experiments, however, we found that simply sharpening the likelihood via tempering or the noisy Dirichlet model is sufficient, and the proper augmentation likelihood does not provide significant benefits in practice. 
In particular, we found the proper augmentation likelihood much harder to sample from with standard SGMCMC samplers, compared to the standard likelihood.
We leave a detailed exploration of the augmentation likelihood for future work.
\citet{nabarro2021data} also derived a valid likelihood for data augmentation: in their model, the predictions are averaged across all data augmentations of a given datapoint.

\section{Experiments}
\label{sec:experiments}

In this section we provide empirical support for the observations presented in \cref{sec:aleatoric_uncertainty_understanding,sec:data_aug_effects}.
First, in \cref{sec:exp_synthetic} we illustrate the effects of tempering, noisy Dirichlet model and data augmentation using a Bayesian neural network on a synthetic 2-D classification problem.
Next, in \cref{sec:data_aug_gp} we visualize the effect of data augmentation on the limiting distribution of SGMCMC using a Gaussian process regression model.
Finally, in \cref{sec:exp_bnn_images} we report the results for BNNs on image classification problems.

\begin{figure*}[!t]
	\centering
	\begin{subfigure}{0.24\linewidth}
		\captionsetup{font=small}
		\includegraphics[height=1.1\textwidth]{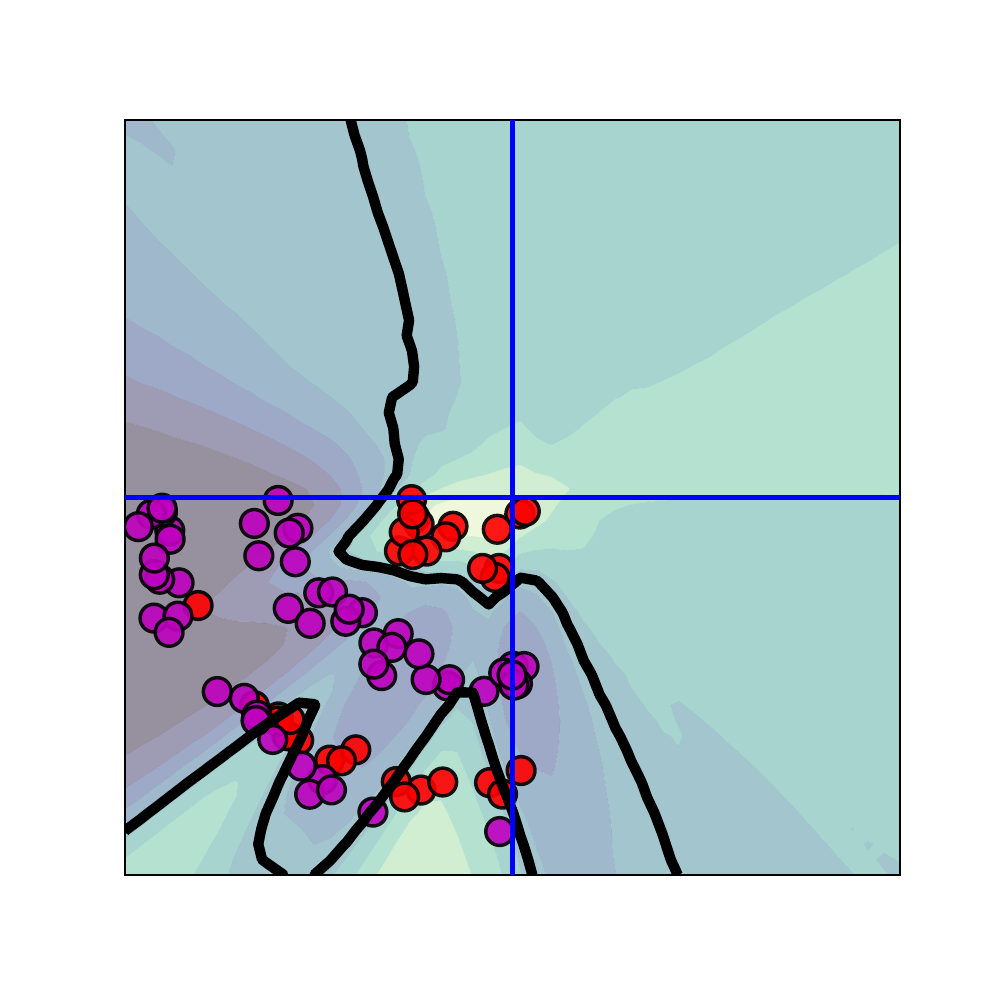}
		\caption{Softmax Lik. (SL)}
		\label{fig:da_crossent}
	\end{subfigure}
	\begin{subfigure}{0.24\linewidth}
		\captionsetup{font=small}
		\includegraphics[height=1.1\textwidth]{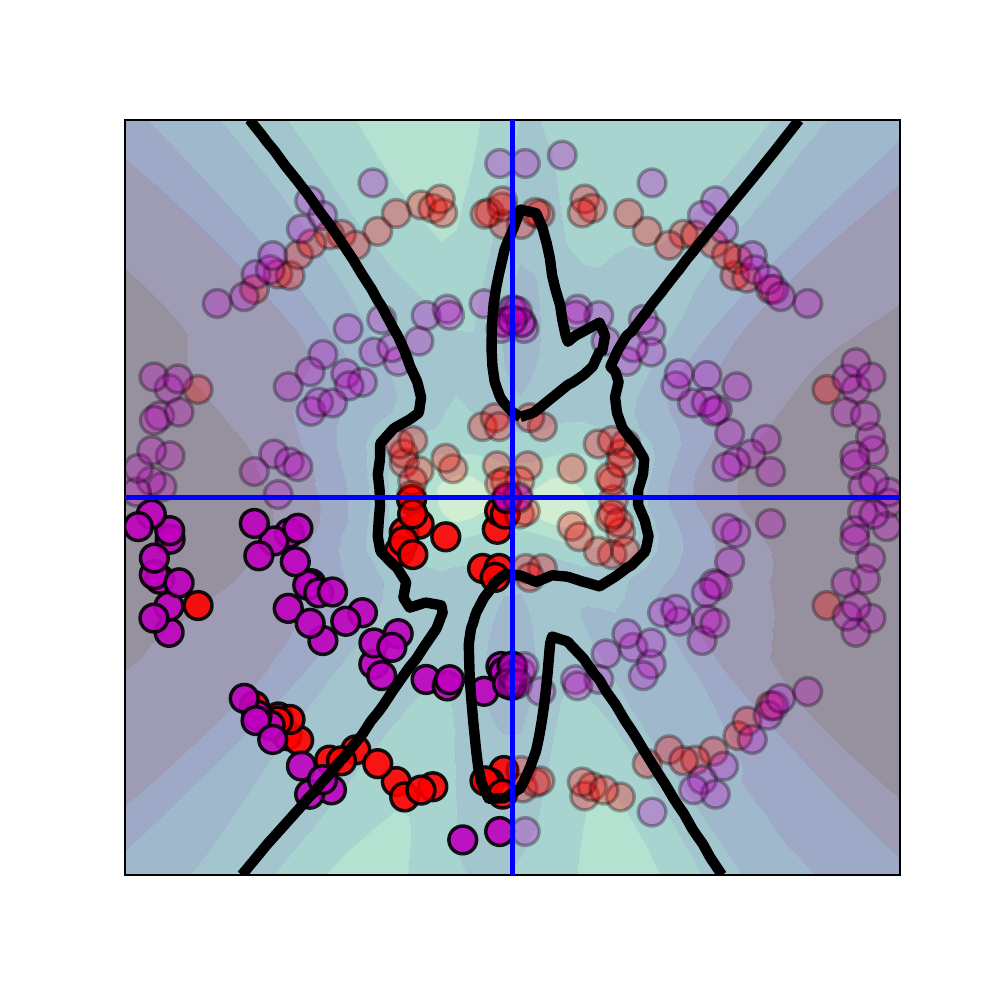}
		\caption{SL+Data Aug.(DA)}
		\label{fig:da_crossent_da}
	\end{subfigure}
	\begin{subfigure}{0.24\linewidth}
		\captionsetup{font=small}
		\includegraphics[height=1.1\textwidth]{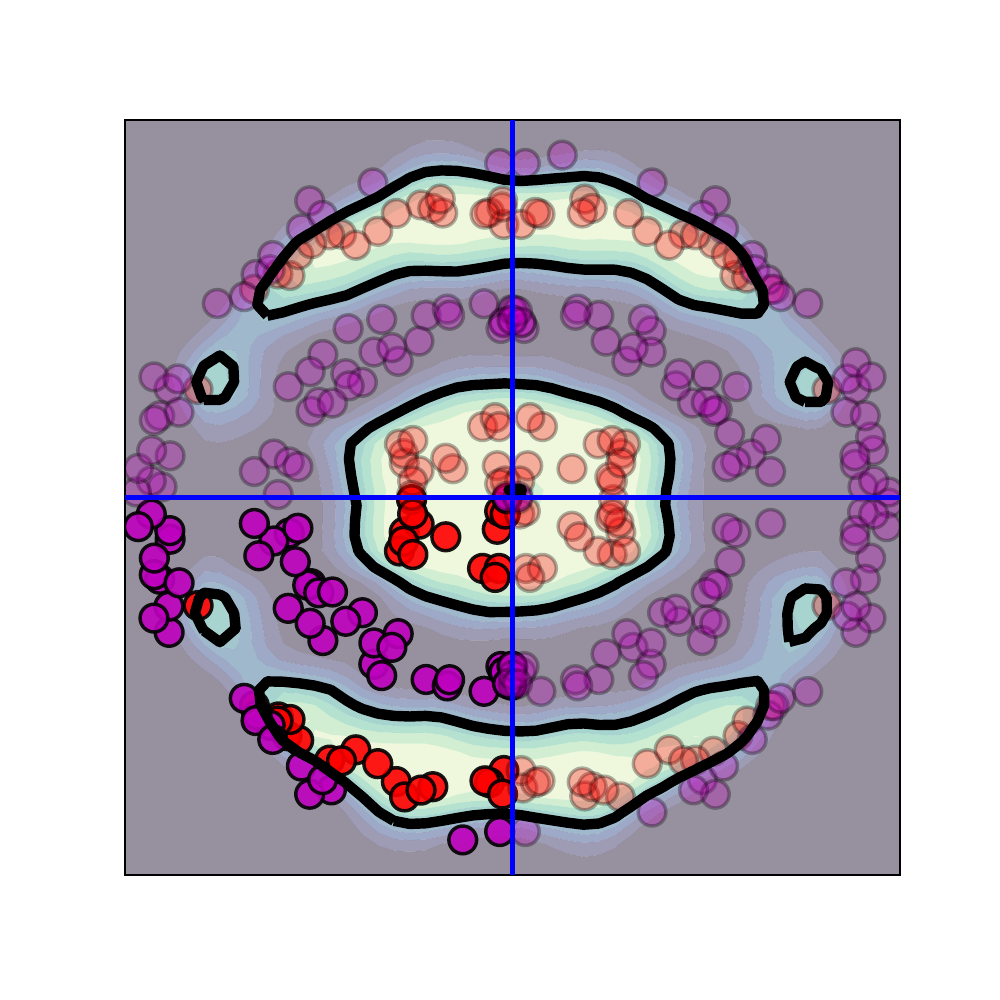}
		\caption{SL+DA+Tempering}
		\label{fig:da_crossent_da_tempering}
	\end{subfigure}
	\begin{subfigure}{0.24\linewidth}
		\captionsetup{font=small}
		\includegraphics[height=1.1\textwidth]{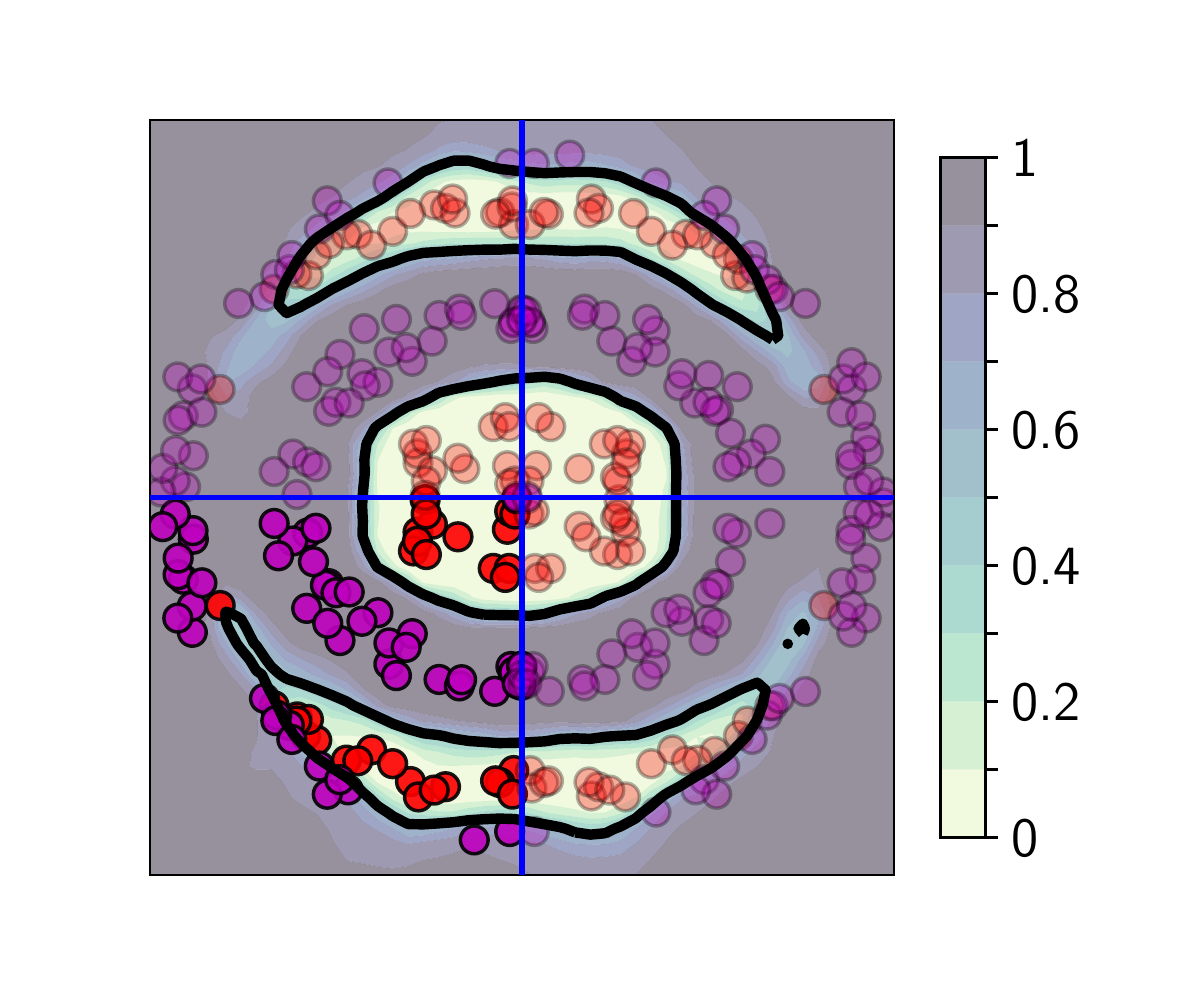}
		\caption{Noisy Dirichlet+DA}
		\label{fig:da_dirichlet_da}
	\end{subfigure}
	\caption{
		\textbf{Insufficient Sharpness of Softmax Likelihood.}
		Decision boundary of BNN classifiers on a synthetic problem.
		In each panel, the decision boundary is shown with a black line, datapoints are shown with magenta (class 0) and red (class 1) circles.
		We consider data augmentation -- random flip about the axes shown with blue lines; the augmented datapoints are shown semi-transparently.
		We use HMC to approximate the posterior in each of the panels.
		\textbf{(a)} A standard softmax likelihood provides a reasonable but imperfect fit of the training data.
		\textbf{(b)} Adding data augmentation leads to a more diffuse likelihood which is underconfident, and therefore an even worse fit.
		\textbf{(c)} Tempering with $T=0.1$ sharpens the softmax likelihood leading to a perfect fit on the train data.
		\textbf{(d)} A noisy Dirichlet model, provides a similar fit without the need for tempering.
	}
	\label{fig:aleatoric_pitfalls}
\end{figure*}

\subsection{Synthetic problem}
\label{sec:exp_synthetic}

We generate the data from a modified version of the two spirals dataset \citep[see e.g.][]{maddox2020rethinking}, where we restrict the data to the $x_1 <0, x_2 <0$ quadrant.
The exact code used to generate the data is available in the \cref{sec:app_exp_synthetic}.

In \cref{fig:aleatoric_pitfalls}(a) we visualize the data and the decision boundary of a Bayesian neural network.
For the Bayesian neural network, we use an iid Gaussian prior $\mathcal N(0, 0.3^2)$ over the parameters of the model;
we use full batch Hamiltonian Monte Carlo (HMC) to sample from the posterior, following \citet{izmailov2021bayesian}.
With the chosen prior, the model is not able to fit the data perfectly, but provides a reasonable fit of the training data.

In \cref{fig:aleatoric_pitfalls}(b) we consider the effect of data augmentation.
Specifically, we apply random flips about the $x_1$ and $x_2$ axes.
In the figure, the augmented datapoints are shown semi-transparently.
To evaluate the network under data augmentation, we run HMC to sample from the posterior distribution in \cref{eq:implied_posterior}.
The Bayesian neural network provides a lower quality fit to the data compared to the fit without data augmentation shown in \cref{fig:aleatoric_pitfalls}(a).
Indeed, according to our analysis in \cref{sec:data_aug_effects}, the observation model is softened by the data augmentation, leading the model to fit the training data poorly.

According to \cref{sec:aleatoric_uncertainty_understanding}, we can sharpen the model leading to a much better fit on the train data by using the tempered likelihood as in \cref{fig:aleatoric_pitfalls}(c) or the noisy Dirichlet observation model as in \cref{fig:aleatoric_pitfalls}(d).
With both of these approaches we achieve a near-perfect fit on the training data, which is desirable if we assume low aleatoric unceratinty.
For further details on this experiment, please see \cref{sec:app_exp_synthetic}.

\subsection{Visualizing the Effect of Data Augmentation}
\label{sec:data_aug_gp}

\begin{figure}[!t]
    \centering
    \includegraphics[width=.9\linewidth]{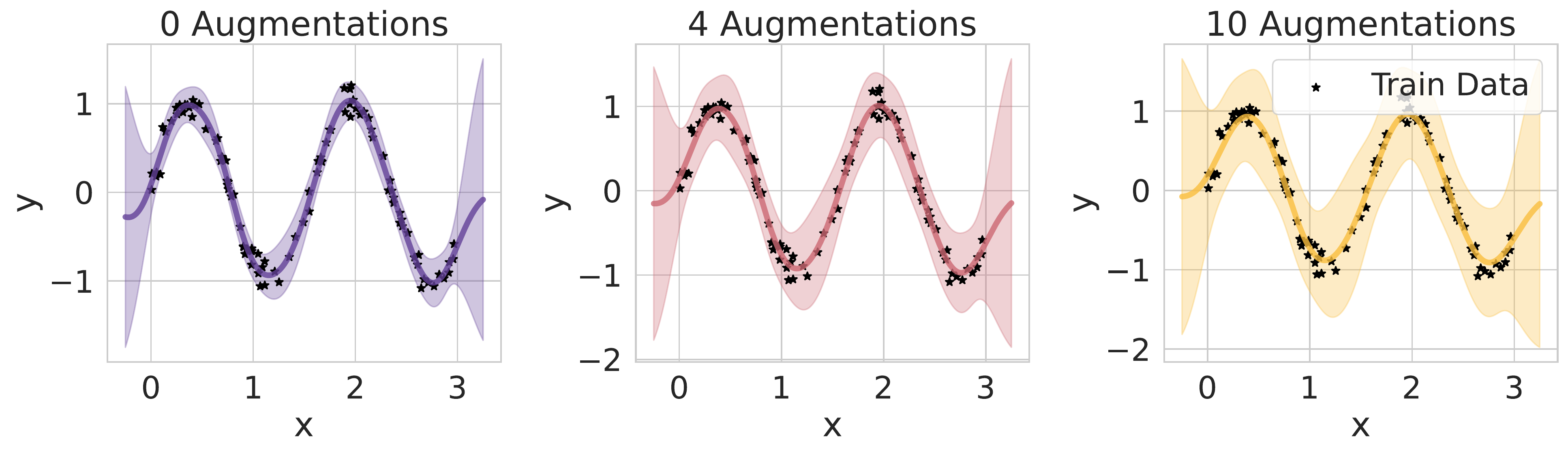}
    \caption{
    \textbf{Effect of data augmentation on GP regression.}
    As we increase the number $K$ of augmentations of the dataset, a GP regression model fit becomes more diffuse.
    Here the original training data is shown with black stars, and the augmented datapoints $t_j(x) = x + \tau_j$ that are obtained by shifting the inputs $x$ by a large constant $\tau_j$ are not shown.
    Note that both the predictive mean and confidence change with the number $K$ of augmentations.
    }
    \label{fig:augmentation_shift_gp}
\end{figure}

To illustrate our analysis in \cref{sec:data_aug_effects} we visualize the posterior in \cref{eq:implied_posterior} under data augmentation in a Gaussian process regression model.
We construct a Gaussian process \citep[GP][]{rasmussen2006gaussian} with a truncated RBF kernel such that $k(x,x^\prime) = 0$ when $\norm{x-x^\prime} > \delta$, where $\delta$ is the lengthscale of the kernel.
Consequently, the implied correlations between inputs beyond the distance $\delta$ are zero.
We use an augmentation policy such that an augmented sample is given by $t_j(x) = x + \tau_j$, for some set of vectors $\{\tau_j\}_{j=1}^K$ such that  $\norm{\tau_j} \gg \delta$.
Therefore we assume that $k(x, t_j(x')) = 0$ for all datapoints $x, x'$ and all augmentations $t_j$.
In other words, the outputs of the Gaussian process on the training data are independent from the outputs on the augmented datapoints. 

We visualize the posterior in \cref{eq:implied_posterior} for $K=1, 4$ and $10$ augmentations per datapoint in \cref{fig:augmentation_shift_gp}.
In the visualization, we show the posterior just on the original data, and do not show the augmented datapoints.
As the predictions on the train data are independent from the predictions on the augmented datapoints, the posterior in \cref{eq:implied_posterior} corresponds to tempering the posterior on the original training data with a warm temperature equal to the number $K$ of augmentations.
As a result, increasing the number of augmentations softens the likelihood, leading to a less confident fit on the training data.
In the next section, we show similar underconfidence on the training data for Bayesian neural networks under data augmentation.

\subsection{Image Classification with Bayesian Neural Networks}
\label{sec:exp_bnn_images}

In this section we experimentally verify the results of the analysis in \cref{sec:aleatoric_uncertainty_understanding,sec:data_aug_effects} for Bayesian neural networks in image classification problems.
First, we show that the noisy Dirichlet model does not require tempering to achieve optimal performance.
Then, we show that both the noisy Dirichlet model and tempered softmax likelihood can be successfully used to express our beliefs about the amount of label noise in the data.
Finally, we show that data augmentation softens the likelihood for BNNs, and that the optimal temperature depends on the complexity of the data augmentation policy.

For all experiments we use a ResNet-18 model \citep{he2016identity} and the CIFAR-10 \citep{Krizhevsky2009LearningML} and Tiny Imagenet \citep{Le2015TinyIV} datasets.
We use the SGLD sampler with a cyclical learning rate schedule \citep{welling2011bayesian, zhang2019cyclical} to sample from the posterior.
We provide details on the hyper-parameters in \cref{app:experimental_details}.

\begin{wrapfigure}{R}{.4\textwidth}
    \vspace{-1em}
    \centering
    \includegraphics[width=.9\linewidth]{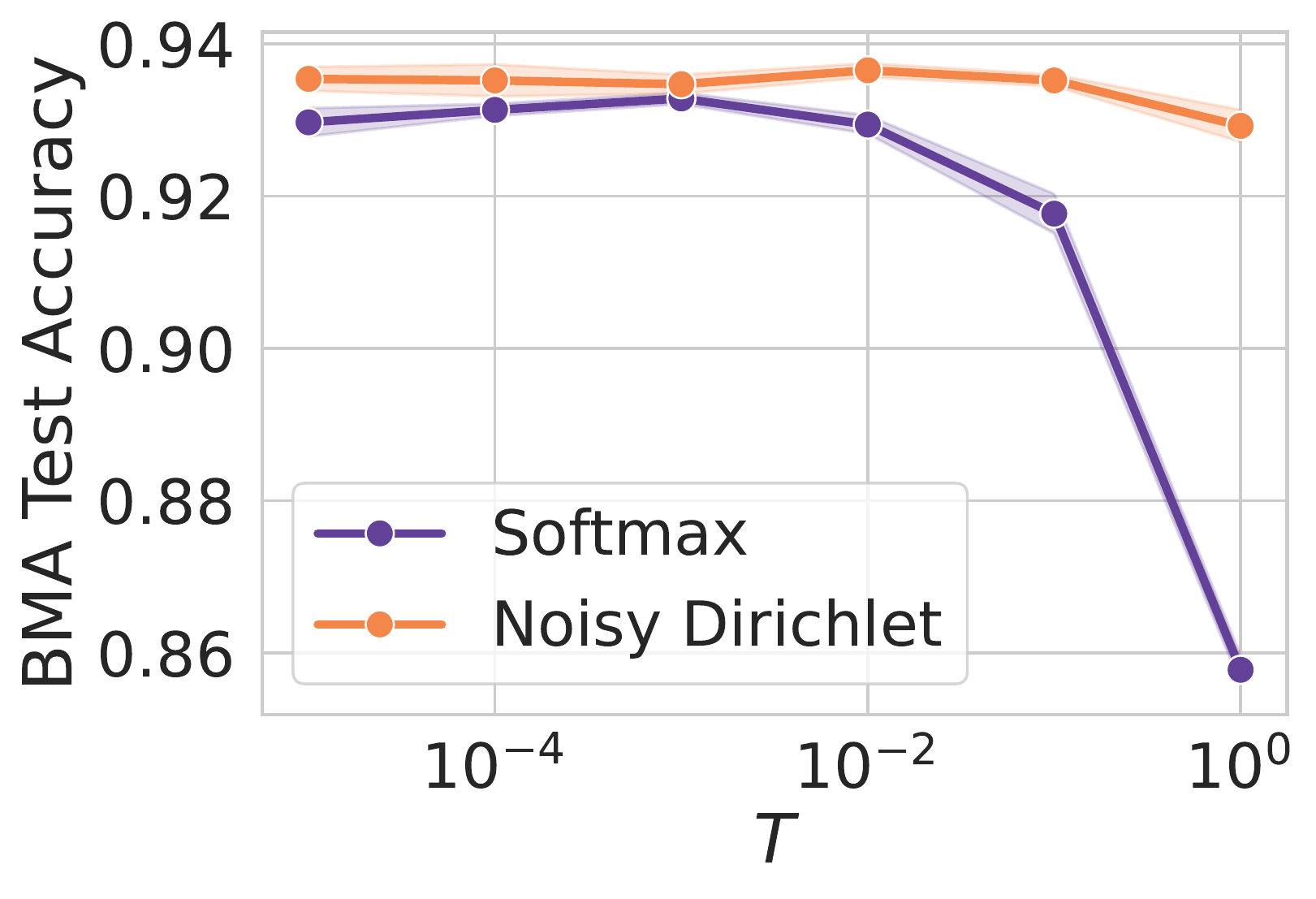}
    \caption{
    BMA test accuracy for the noisy Dirichlet model with noise parameter $\alpha_\epsilon = 10^{-6}$ and the softmax likelihood
    as a function of posterior temperature on CIFAR-10.
    The noisy Dirichlet model shows no cold posterior effect.
    }
    \label{fig:no_cpe}
\end{wrapfigure}
\textbf{No cold posterior effect in the noisy Dirichlet model.}
\quad
First, we explore the effect of posterior tempering on the standard softmax classification likelihood and the noisy Dirichlet model with noise parameter $\alpha_\epsilon=10^{-6}$.
In \cref{fig:no_cpe} we show the results on the CIFAR-10 dataset.
As reported by \citet{wenzel2020good}, for the standard softmax likelihood, tempering is required to achieve optimal performance, with $T=10^{-3}$ providing the best results.
For the noisy Dirichlet model on the other hand, tempering does not significantly improve the results with roughly constant performance across different temperature values.
In particular, the noisy Dirichlet model achieves near-optimal results at temperature $1$!
These results agree with our analysis in \cref{sec:aleatoric_uncertainty_understanding}: both tempering and the noisy Dirichlet observation model are alternative ways of expressing our beliefs about the aleatoric uncertainty in the data;
with the noisy Dirichlet model, we can achieve strong results without any need for tempering.

\begin{figure}[!ht]
\centering
\begin{tabular}{cc}
    \includegraphics[width=.47\linewidth]{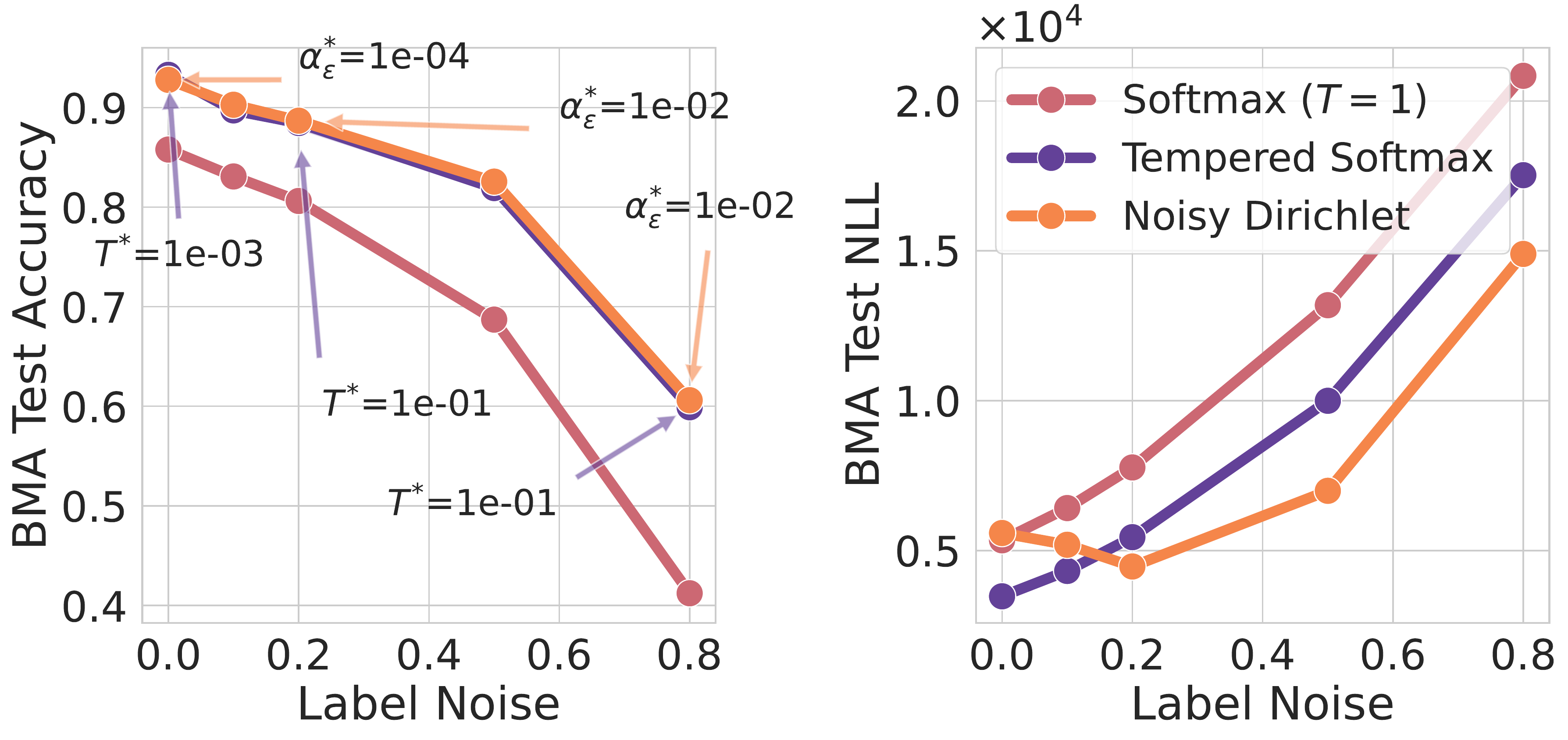}
    &
    \includegraphics[width=.47\linewidth]{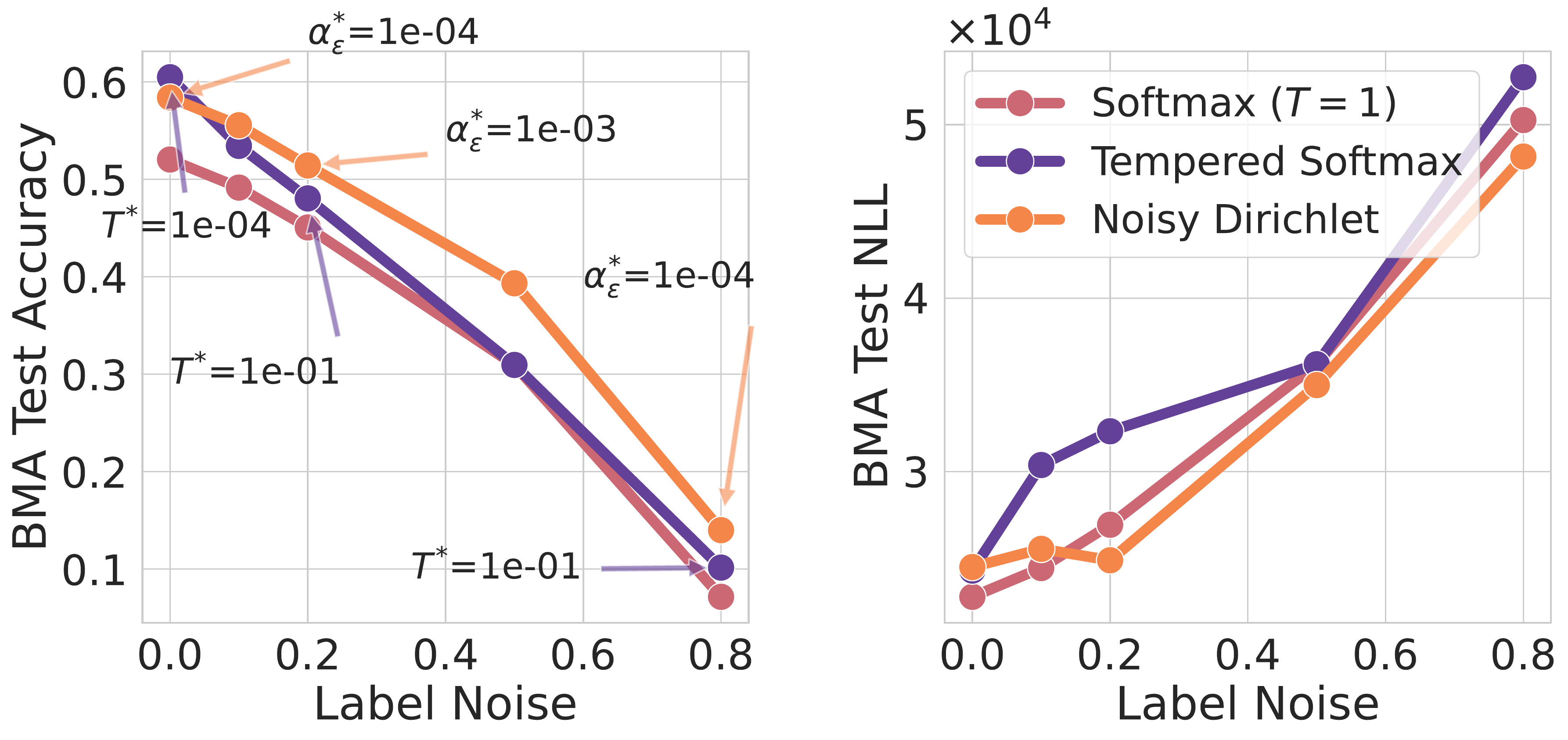}
\\
   (a) CIFAR-10 & (b) Tiny Imagenet
\end{tabular}
\caption{
\textbf{Label noise in BNN image classification.}
BMA test accuracy and negative log likelihood for the standard softmax, tempered softmax and noisy Dirichlet model 
\textbf{(a)} on CIFAR-10 and \textbf{(b)} Tiny Imagenet.
Accounting for the label noise via the noisy Dirichlet model or the tempered softmax likelihood significantly improves accuracy accross the board.
Moreover, the optimal performance is achieved by different values of temperature $T^{*}$ in the tempered softmax likelihood or noise $\alpha_\epsilon^{*}$ parameter in the noisy Dirichlet model, i.e.
no one model can describe the data optimally across all levels of label noise.
}
\label{fig:noisy_dirichlet_perf_cifar_labelnoise}
\end{figure}

\textbf{Modeling label noise.}\quad
In classification problems the aleatoric uncertainty corresponds to the amount of label noise present in the data.
Throughout the paper we argued that aleatoric uncertainty is misrepresented by standard Bayesian neural networks, and that
likelihood tempering and the noisy Dirichlet model are compelling approaches to incorporate information about the amount of label noise.
In \cref{fig:noisy_dirichlet_perf_cifar_labelnoise}, we show the BMA test accuracy and negative log-likelihood for the standard softmax likelihood, tempered softmax likelihood and the noisy Dirichlet model on CIFAR-10 and Tiny Imagenet under varying amounts of label noise.
We plot the results for the best performing temperature $T \in \{10^{-5},10^{-4},10^{-3},10^{-2},10^{-1},1,3,10\}$ for the standard softmax likelihood or noise $\alpha_\epsilon \in \{ 10^{-6}, 10^{-5}, 10^{-4}, 10^{-3}, 10^{-2}, 10^{-1}  \}$ for the noisy Dirichlet model.
Both on CIFAR-10 and on Tiny Imagenet and across all values of label noise, we can significantly improve performance over the standard softmax likelihood by explicitly modeling the aleatoric uncertainty with either tempering or the noisy Dirichlet model.
Moreover, different amounts of label noise require different values of the temperature or $\alpha_\epsilon$ parameter.

In \cref{fig:noisy_dirichlet_perf_cifar_labelnoise}(b), for Tiny Imagenet the noisy Dirichlet model provides a significant improvement in performance compared to the tempered softmax likelihood for high levels of label noise.
Consequently, one may wonder whether the noisy Dirichlet model should become the default for Bayesian classification.
Generally, no one approach is optimal across the board: we should pick the model that best desribes our beliefs about the data, which may be expressed by tempering the softmax likelihood, the noisy Dirichlet model or another approach.

\textbf{Data augmentation leads to underfitting on train data.}
\quad
Next, we verify the results from \cref{sec:data_aug_effects} for Bayesian neural networks in image classification.
In \cref{fig:train_diffuse_aug_lik}(a,b) we show the BMA train negative log-likelihood for the models trained with and without data augmentation.
Across a wide range of temperatures for the tempered softmax and $\alpha_\epsilon$ parameters for the noisy Dirichlet model,
adding the data augmentation reduces the quality of the fit on the original training data.
This result is analogous to the results presented in \cref{sec:data_aug_gp} and agrees with our analysis in \cref{sec:data_aug_effects}: data augmentation softens the likelihood, leading to a more diffuse fit.

\textbf{Complex augmentations require lower temperatures.}
\quad
Finally, we explore the effect of the number $K$ of data augmentations applied to each datapoint.
In \cref{sec:data_aug_effects} we showed that under data augmentation the likelihood is effectively tempered with a temperature $K$, assuming the predictions on the augmentated datapoints are completely independent from the original datapoints.
In practice however, the augmentations can be close to the original images, and the corresponding predictions can be highly correlated.
In \cref{fig:train_diffuse_aug_lik}(c), we consider five separate types of augmentations for ResNet-$18$ on CIFAR-$10$ at various posterior temperatures:
horizontal and vertical flips, random crops, combinations of flips and crops, and  AugMix \citep{hendrycks2019augmix} --- an augmentation policy employing a very diverse set of augmentations.
We find that, as predicted by the analysis in \cref{sec:data_aug_effects}, the optimal temperatures are different for different augmentation policies:
in terms of test NLL, the simple policies (vertical and horizontal flips) corresponding to $K = 2$ work best at warmer temperatures $T=1$, 
intermediate policies (crops and crops$+$flips) corresponding to $K \approx 100$ require lower temperatures $T \in [10^{-2}, 10^{-1}]$,
and the most complex AugMix policy requires the lowest temperature $T \le 10^{-4}$.

\begin{figure}[!t]
\centering
\begin{tabular}{cc}
    \includegraphics[width=0.48\linewidth]{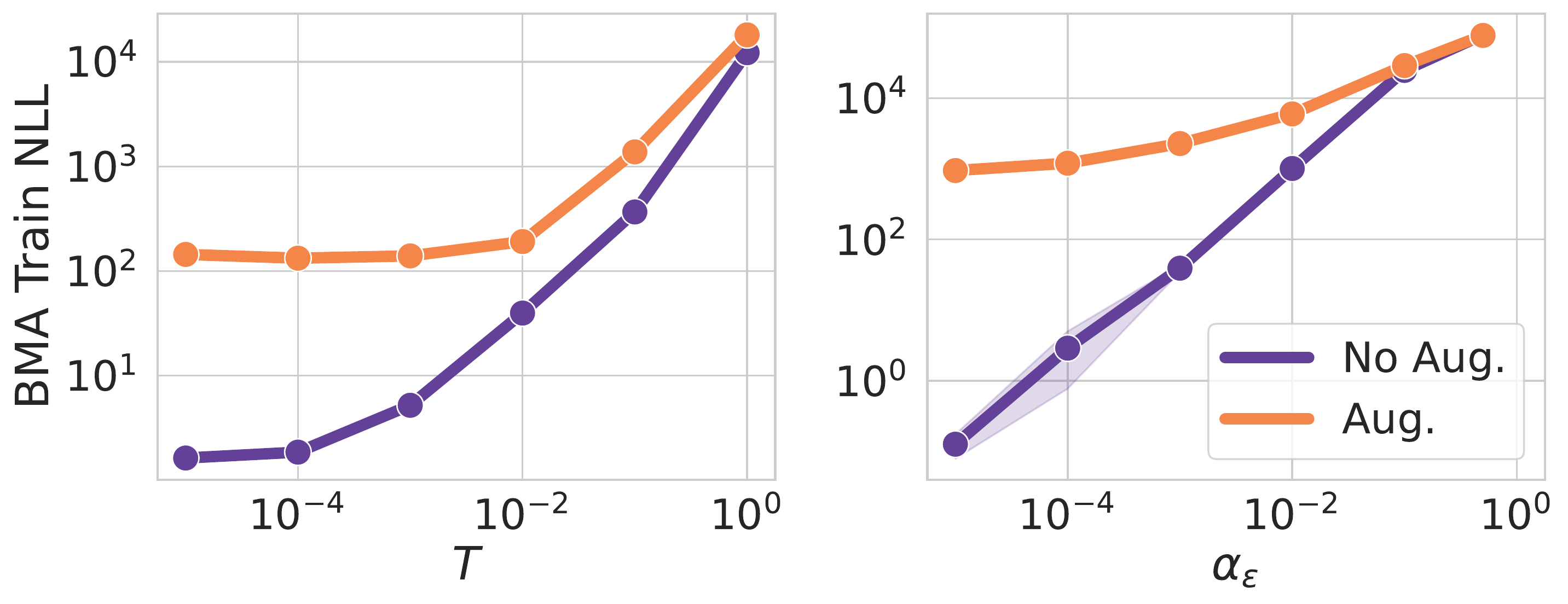} &
    \includegraphics[width=0.49\linewidth]{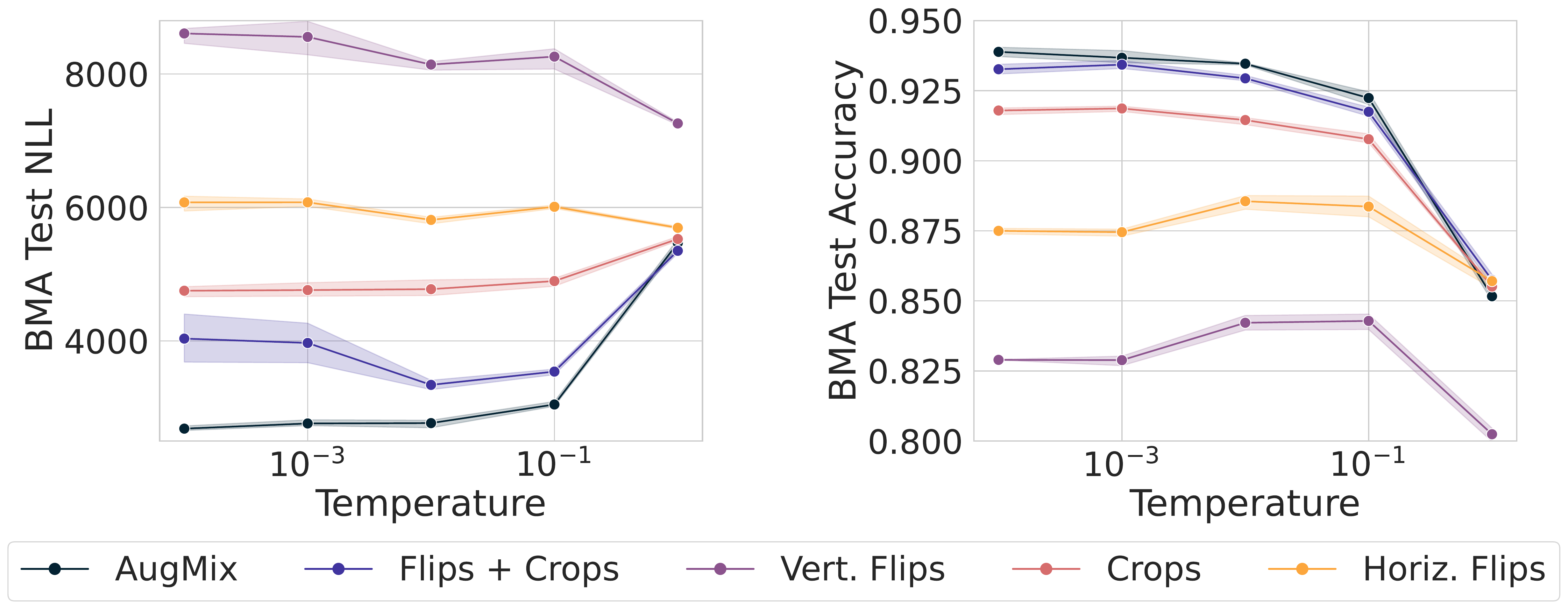}
    \\
    (a) Softmax \& Noisy Dirichlet & (b) Augmentations vs Tempering
\end{tabular}

\caption{
\textbf{Effect of data augmentation on BNN image classification.}
For all plots we use ResNet-18 on the CIFAR-10 dataset.
\textbf{(a)}: The BMA negative log-likelihood on the train data for the tempered softmax model across different temperatures $T$ and the noisy Dirichlet model across different values of noise $\alpha_\epsilon$, both with and without data augmentation.
As predicted in \cref{sec:data_aug_effects}, data augmentation softens the likelihood, leading to a more diffuse fit on the train data.
Tempering or reducing the noise parameter $\alpha_\epsilon$ is needed to counteract the effect of data augmentation.
\textbf{(b)} BMA test accuracy and NLL for various augmentation policies.
As predicted in \cref{sec:data_aug_effects}, more complex policies corresponding to a higher effective number of augmentations $K$ require lower temperatures for optimal performance.
}
\label{fig:train_diffuse_aug_lik}
\end{figure}

\section{Discussion}

A correct representation of aletoric uncertainty is crucial to achieve strong performance with Bayesian neural networks (BNNs).
Standard Bayesian classifiers have no explicit mechanism to represent our beliefs about aleatoric uncertainty, which can lead to 
inadequate fits to the training data.
This effect is exaggerated by data augmentation, which softens the likelihood, making the models underconfident on the training data.
Posterior tempering is a simple and practical way to correct for this misspecification and express our beliefs about the aleatoric uncertainty:
most benchmark datasets have very low levels of label noise, and we can express this belief by tempering the softmax likelihood.
We also showed that we can achieve a similar effect without tempering, by using a prior that forces the model to be confident on the training datapoints with the noisy Dirichlet model.

For practitioners, for any classification problem with Bayesian neural networks, we recommend to use one of the mechanisms we discussed to incorporate beliefs 
about aleatoric uncertainty. We have found both tempered softmax likelihood and noisy Dirichlet model highly effective both for data with low and high levels of label noise.

This work is generally about representing aleatoric uncertainty in Bayesian classification, and precisely how data augmentation counterintuitively reduces
the confidence in the observation labels. However, we also resolve many observations about the cold posterior effect in Bayesian neural networks, 
which we now discuss in the context of recent work.

\textbf{\citet{Noci2021DisentanglingTR}: Data augmentation is sufficient but not necessary for CPE.}
\citet{Noci2021DisentanglingTR} confirm the findings in \citet{izmailov2021bayesian} and \citet{fortuin2021bayesian} that data augmentation plays a significant role in the cold posterior effect.
In particular, \citet{izmailov2021bayesian} show that removing data augmentation from the models in \citet{wenzel2020good} alleviates the cold posterior effect in all of their examples.
However, \citet{Noci2021DisentanglingTR} also show that it is sometimes possible to achieve CPE without data augmentation, for example by subsampling the dataset. Similarly, 
\citet{wilson2020bayesian} note that many types of misspecification could lead to a cold posterior effect.

In our work, we provide a more nuanced perspective on the cold posterior effect and data augmentation:
data augmentation leads to a poor representation of aleatoric uncertainty, which can be addressed by tempering the posterior.
We note that BNNs can misrepresent aleatoric uncertainty even without data augmentation, and can have a reasonable representation of aleatoric uncertainty in the presence of data augmentation.
However, in the vast majority of practical scenarios where the cold posterior effect has been demonstrated, it appears due to the use of data augmentation. Thus data augmentation is 
\emph{largely} responsible for the cold posterior effect, and we have shown here that this is precisely because augmentation changes our representation of aleatoric uncertainty to be 
underconfident, and tempering can correct for this underconfidence, more honestly representing our beliefs.

Finally, \citet{izmailov2021dangers} show that BNNs have issues under covariate shift, due to the posterior not contracting sufficiently along some directions in the parameter space.
The same issue occurs when BNNs are applied to extremely small datasets, which may affect the results on data subsampling presented in \citet{Noci2021DisentanglingTR}. 

\textbf{\citet{nabarro2021data}: the lack of a CPE without DA is an artifact of using the wrong model (i.e. without DA).}\quad
Similar to our discussion in \cref{sec:data_aug_effects}, \citet{nabarro2021data} point out that tempering does not always provide a principled approach to data augmentation.
They instead devise a new observation model, where the model outputs are averaged over all the augmented datapoints.
With this observation model, they still observe the cold posterior effect.
These results are interesting in that they show that a valid likelihood is not sufficient to remove the cold posterior effect.
However, these results do not rule out the possibility that the proposed likelihood is a poor description of the data generation process: indeed, in reality the labels for the training image are not produced by considering all the possible augmentations of these images.

Generally, our results suggest that the cold posterior effect is caused by a poor representation of the aleatoric uncertainty in the BNNs.
It can happen with or without data augmentation, with valid or invalid likelihoods. A valid likelihood is not necessarily a well-specified likelihood.
However, we believe that the interpretation that the lack of CPE in models without data augmentation is an artifact of miss-specification is highly questionable:
when our beliefs about the aleatoric uncertainty in the data are correctly captured by the model, we should not need posterior tempering to achieve strong results;
it is the need for tempering which is an indication of miss-specification.

\textbf{\citet{Noci2021DisentanglingTR}: Cold posteriors are unlikely to arise from a single simple cause.}\quad
\citet{Noci2021DisentanglingTR} show examples, where the cold posterior effect arises in isolation from data augmentation, data curation, or prior misspecification.
Consequently, they argue that no one cause can fully explain the cold posterior effect.
While we agree with this general argument, which is also aligned with the discussion in \citet{wilson2020bayesian}, we note that all of the considered causes are directly related to aleatoric uncertainty.
Indeed, in this work we have shown that data augmentation significantly affects the level of aleatoric uncertainty assumed by the model.
The data curation is also directly connected to aleatoric uncertainty, as curated datasets are expected to have low label noise.
Finally, in BNNs the prior over the parameters specifies the assumptions about aleatoric uncertainty, and tempering can be used to correct for the effect of the prior.
We believe that our results in this paper provide a compelling explanation for the observations in \citet{Noci2021DisentanglingTR}.

\textbf{Summary.} Overall, properly representing aleatoric uncertainty is a challenging but fundamentally important consideration in Bayesian classification. We have shown that posterior tempering provides a mechanism
to more honestly represent our beliefs about aleatoric uncertainty, especially in the presence of data augmentation. In general, as in \citet{wilson2020bayesian}, we should not be alarmed if $T=1$ is not 
optimal in sophisticated models on complex real-world datasets. Moreover, we have shown how other mechanisms to represent aleatoric uncertainty, such as the noisy Dirichlet model, do not suffer from a
cold posterior effect in the presence of data augmentation. Indeed, while an interesting phenomenon, cold posteriors should not be conflated with the success or failure of Bayesian deep learning. In general, 
approximate inference in Bayesian neural networks has been making great strides forward, often providing better performance at a comparable computational cost to standard methods. However, there
\emph{are} practical challenges to the adoption of Bayesian deep learning. For example, \citet{izmailov2021dangers} shows that Bayesian neural networks can profoundly degrade in performance under
a wide range of relatively minor distribution shifts --- behaviour which could affect applicability on virtually any real-world problem, since train and test rarely come from exactly the same distribution. While their \emph{EmpCov} prior provides a partial remedy, there is still much work to be done. For example one could develop priors that protect against covariate shift by accounting for linear dependencies in the internal representations of the network.

\textbf{Acknowledgments.} 
We would like to thank Micah Goldblum and Wanqian Yang for helpful comments. This research is supported by an Amazon Research Award, NSF I-DISRE 193471, NIH R01DA048764-01A1, NSF IIS-1910266, and NSF 1922658 NRT-HDR: FUTURE Foundations, Translation, and Responsibility for Data Science.
We also thank NYU IT High Performance Computing for providing the infrastructure to run our experiments.

\vspace{-5mm}
\bibliography{references}

\clearpage
\appendix

\section{Detailed Description of the Synthetic Experiment}
\label{sec:app_toy_exp}

To intuitively grasp the impact of aleatoric uncertainty on inference, we
first develop a toy illustration. 
Consider the \textit{two spirals} binary classification problem \citep{huang2020understanding} where,
for training we generate $50$ samples with both $x_1,x_2 < 0$.
\cref{fig:aleatoric_pitfalls} visualizes the posterior predictive density for various settings.

By naively learning from the training data using the softmax likelihood, we find that 
the predictive density is poorly calibrated to the effect that it does not learn
anything meaningful except in the region containing the training data, as 
visualized in \cref{fig:da_crossent}. But, by virtue of our a priori knowledge
about the mirror symmetry across both axes, we can create augmented data, i.e.
for every $(x,y)$ pair in the training data, we create augmentations 
$(\pm x, \pm y)$ for use in training. With this modification, we find that the predictive density surface in \cref{fig:da_crossent_da} appears to improve but remains very diffuse, implying underconfidence.

But in fact, we can fix the predictive density surface (and hence the performance) by
accounting for aleatoric uncertainty. In \cref{fig:da_crossent_da_tempering} we see that
tempering allows us to sharpen the softmax likelihood such that the final predictive
density does not remain diffuse. We are implicitly correcting for aleatoric
uncertainty which was artificially inflated due to data 
augmentations. Alternatively, by using a noisy Dirichlet likelihood which allows
explicit control over the amount of aleatoric uncertainty in our data observation
model, we are able to achieve a similar effect as shown in \cref{fig:da_dirichlet_da}.

This toy illustration is representative of how modern machine learning models for
classification are trained in practice, and provide a first direct understanding
of the interaction of aleatoric uncertainty with tempering and data augmentation.

\section{Sharpening with the Smoothed Softmax Likelihood}
\label{sec:sharp_logits}

A natural approach to consider sharpening the softmax likelihood is by smoothing the logits. The \emph{tempered softmax likelihood} for a given class observation $y=c$ 
conditional on input $x$ is given by,
\begin{align}
\frac{1}{T} \log{p(y = c \mid x)} = \log \frac{\exp\{ f_c(x; \params) / T\}}{(\sum_{j=1}^C \exp\{f_j(x; \params)\})^{1/T}} \label{eq:tempered_softmax}
\end{align}
where $f_j(x; \params) \in \reals$ represents the logit for the $j^{\mathrm{th}}$ class of a total of $C$ classes. A
properly normalized distribution, however, leads to what we term the 
\emph{smoothed softmax likelihood},
\begin{align}
\log{p_{T}(y = c \mid x)} =  \log \frac{\exp\{ f_c(x; \params) / T\}}{\sum_{j=1}^C \exp\{f_j(x;\params)/T\}}. \label{eq:smoothed_softmax}
\end{align}
Versions of this likelihood have appeared in prior literature, especially when 
considering distillation \citep{hinton2015distilling} and temperature (e.g. Platt) scaling \citep{platt1999probabilistic,guo2017calibration}. 
More recently, \citet{zeno2020cold} have pointed out a connection between \cref{eq:smoothed_softmax} and the cold posteriors \cref{eq:cold_posterior} studied by \citet{wenzel2020good}.

\begin{wrapfigure}{R}{.333\textwidth}
\includegraphics[width=\linewidth]{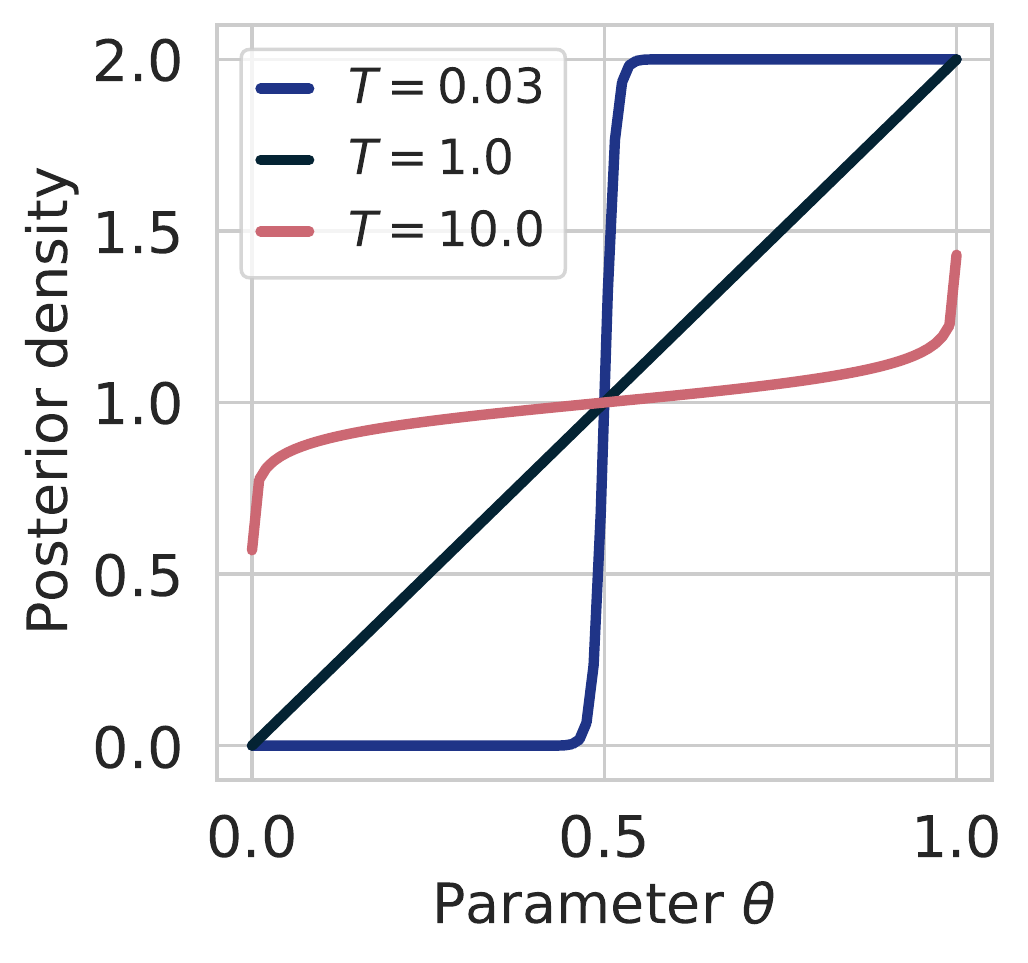}
\centering
\caption{A smoothed likelihood \cref{eq:smoothed_softmax} is insufficient to achieve arbitrary sharpness (\cref{sec:sharp_logits}). An arbitrarily sharp likelihood would concentrate around a single value of $\theta$ rather than being effectively uniform over ranges.}
\label{fig:logit_scaling}
\vspace{-1.5em}
\end{wrapfigure}

\paragraph{Insufficiency of Smoothed Softmax}
Consider a toy coinflip example --- we are interested in
inferring the probability of heads $\params \in [0,1]$ using a single observation of
a heads $\dset = \{H\}$. The posterior over $\params$, under the assumption of a uniform prior ${p(\params) = \mathrm{U}[0,1]}$ is given by ${p(\params \mid \dset) \propto p(H \mid \params) = \theta}$. The smoothed softmax likelihood corresponds to rescaling the logit ${\ell(\params) = \log\{\params/(1-\params)}\}$ as ${p_T(H \mid \params) = \sigma(\ell(\params) / T)}$, where ${\sigma(z) = 1/(1+\exp{\{-z\}})}$ is the sigmoid function  \footnote{Note here that $\sigma(\ell(\theta)) = \theta$}.

Using the tempered softmax likelihood by raising to a power $1/T$ where $T<1$, we
effectively sharpen the density function and force the posterior to concentrate more
sharply around $1$ owing to the single observation. On the other hand, using a 
smoothed softmax likelihood ends up spreading the posterior mass very differently, to 
the effect that it undermines the evidence from observation. This is visualized in 
\cref{fig:logit_scaling}. In fact, no matter what temperature we use for a tempered 
posterior, the smoothed softmax likelihood cannot be made arbitrarily sharp.

In \cref{fig:smooth_vs_tempered_softmax}, we empirically verify with ResNet-$18$ on CIFAR-$10$ that no matter the value of temperature $T$, a smoothed softmax likelihood does not
outperform a tempered likelihood. Moreover, it does not mitigate the cold posterior effect. Therefore, we can return to our discussion of cold and tempered posteriors.

\section{Connecting Cold Posteriors and Tempered Likelihood Posteriors}
\label{sec:conn_cold_tempered}

\subsection{Connections in Bayesian Linear Models}
\label{sec:conn_lin_reg}

We begin by taking a look at the connection between the Bayesian, cold, and tempered posteriors under linear models. For a dataset $\dset = \{X,y\}$ consisting of $d$-dimensional inputs ${X = [x_1;x_2;\dots;x_N] \in \reals^{N \times d}}$, corresponding 
outputs ${y \in \reals^{N}}$, a known observation noise variance $\sigma^2$, and prior
over parameters $\params \sim \gaussian{0, \alpha \mbf{I}}$, the posterior over parameters is given by,
\begin{align}
p(\params \mid \dset) = \mathcal{N} ( & \left(X^\top X + \alpha^{-1} \sigma^2 \mbf{I}\right)^{-1} X^\top y, 
    \left(\sigma^{-2} X^\top X + \alpha^{-1} \mbf{I}\right)^{-1} ).	
\end{align}
Following \citet{grunwald2017inconsistency},who study $T > 1$ in the context of model mis-specfication, the tempered posterior for an explicit temperature $T$ is,
\begin{align}
p_{\mathrm{temp}}(\params \mid \dset) = \mathcal{N} ( & \left( \tfrac{1}{T} X^\top X + \alpha^{-1} T\sigma^2 \mbf{I}\right)^{-1} X^\top y, 
    \left( \tfrac{1}{T}\sigma^{-2} X^\top X + \alpha^{-1} \mbf{I}\right)^{-1} ).
\end{align}
Further, for a \emph{cold posterior}, where both the likelihood and the prior are
raised to a temperature, we can follow \citet{aitchison2020statistical} and rewrite 
the prior over $\params$ as $\gaussian{0, \alpha T  \mbf{I}}$, giving the posterior as,
\begin{align}
p_{\mathrm{cold}}(\params \mid \dset) = \mathcal{N} ( & \left(X^\top X + \alpha^{-1} \sigma^2 \mbf{I}\right)^{-1} X^\top y, 
    T\left(\sigma^{-2} X^\top X + \alpha^{-1} \mbf{I}\right)^{-1} ).	
\end{align}
Thus, ``cold" posteriors in this setting reproduce the same posterior mean as the
standard Bayesian posterior, only varying the value of the posterior covariance 
matrix. We visually demonstrate this effect for the posterior confidence ellipses over $\theta$ in \cref{fig:tempering_explainer} left.

\begin{figure}[!ht]
    \centering
    \includegraphics[width=.7\linewidth]{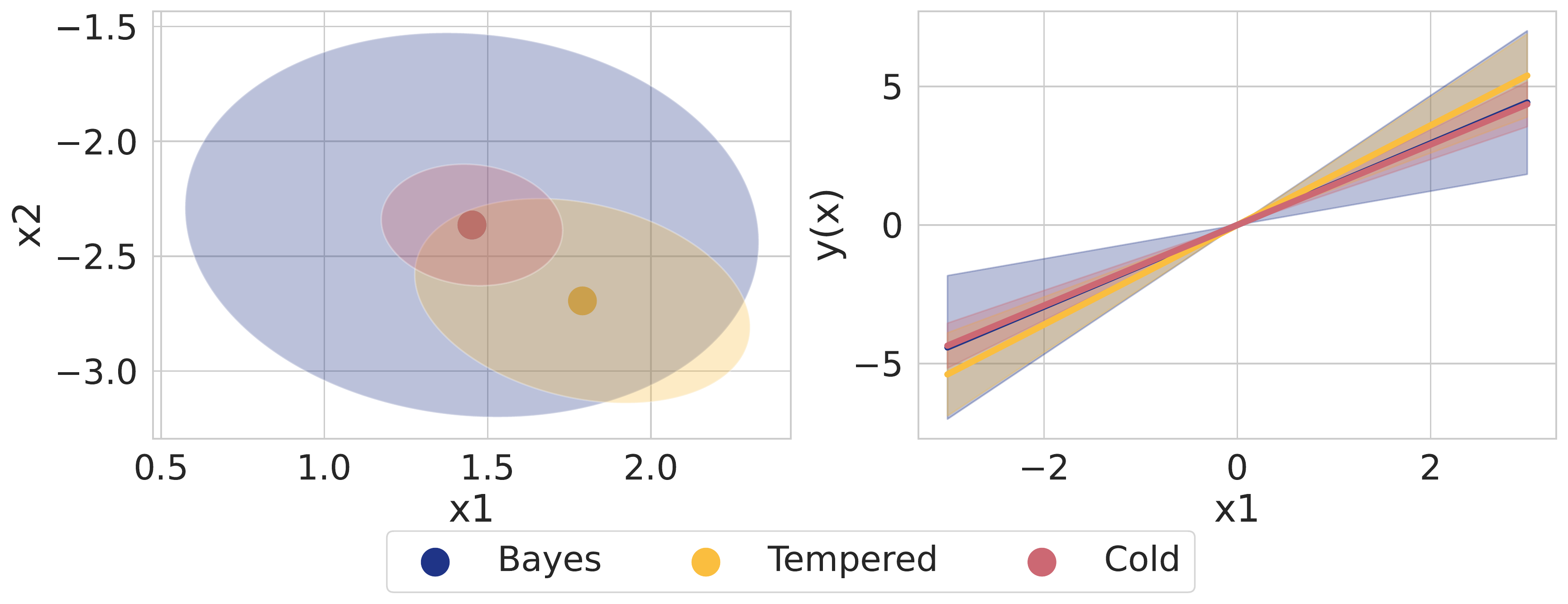}
    \caption{\textbf{Left:} Posterior confidence ellipses for Bayesian linear regression in both Bayesian, tempered, and cold frameworks. \textbf{Right:} Posterior functions for all three approaches. The cold posterior simply rescales the prior, producing more confident fits.}
    \label{fig:tempering_explainer}
\end{figure}

For $T < 1$, we tend to increase the concentration of the posterior around the 
posterior mean. In contrast, for $T > 1$, we tend to decrease the concentration of the 
posterior around the posterior mean. Propagating the posterior distribution on 
$\params$ to the resulting distribution over the function ${y_\star = \params^\top x_\star}$ for novel inputs $x_\star$, the cold posterior over $y_\star$ has the same mean as the Bayesian posterior, while only the posterior covariance is rescaled by a factor of $T$.
We also demonstrate this effect in \cref{fig:tempering_explainer} right, showing that the slope of the cold model is the same as the slope of the fully Bayesian model, only their posterior variances are different. 
However, the slope of the tempered model is different, owing to an effectively increased noise term.

A similar finding is produced by \citet{adlam2020cold} in the context of Gaussian process regression, where similarly cold posteriors simply rescale the posterior variance of the GP posterior, while tempering the posterior tends to increase the noise level.

\subsection{Cold Posteriors and Sharp Likelihoods}
\label{sec:cp_and_sharpness}

While constructing a cold posterior \cref{eq:cold_posterior} involves tempering both the likelihood and the prior, the cold posterior effect still remains even with a tempered likelihood posterior \cref{eq:tempered_posterior} alone. 
Many classes of prior distributions, specifically any prior that is bounded almost 
everywhere, continue to be proper prior distributions when raised to a power $1/T$.
We formalize this statement in \cref{thm:rescaled_priors}.

\begin{theorem}\label{thm:rescaled_priors}
For proper prior distributions, $p(\params),$ that have bounded density functions, e.g. $p(\params) \leq M$ almost everywhere, then for $T \leq 1,$ $p(\params)^{1/T} / \mathcal{I}$ is a proper prior distribution, where $\mathcal{I} = \int p(\theta)^{1/T} d\theta < \infty.$
\end{theorem}
\begin{proof}
As $p(\theta)$ is bounded above by $M$, then $p(\theta) / M \leq 1$ for all $x.$ As $1/T \geq 1,$ then $(p(\theta) / M)^{1/T} \leq p(\theta) / M \leq 1.$
By construction $p(\theta) / M$ is integrable as $p(\params)$ is integrable, so then $(p(\params)/M)^{1/T}$ must also be integrable implying that $p(\params)^{1/T}$ is integrable. 
To make $p(\params)^{1/T}$ a distribution, we only need to construct a normalization constant $\mathcal{I}$.
\end{proof}

\begin{wrapfigure}{R}{.4167\textwidth}
    \vspace{-1em}
    \centering
    \includegraphics[width=.85\linewidth]{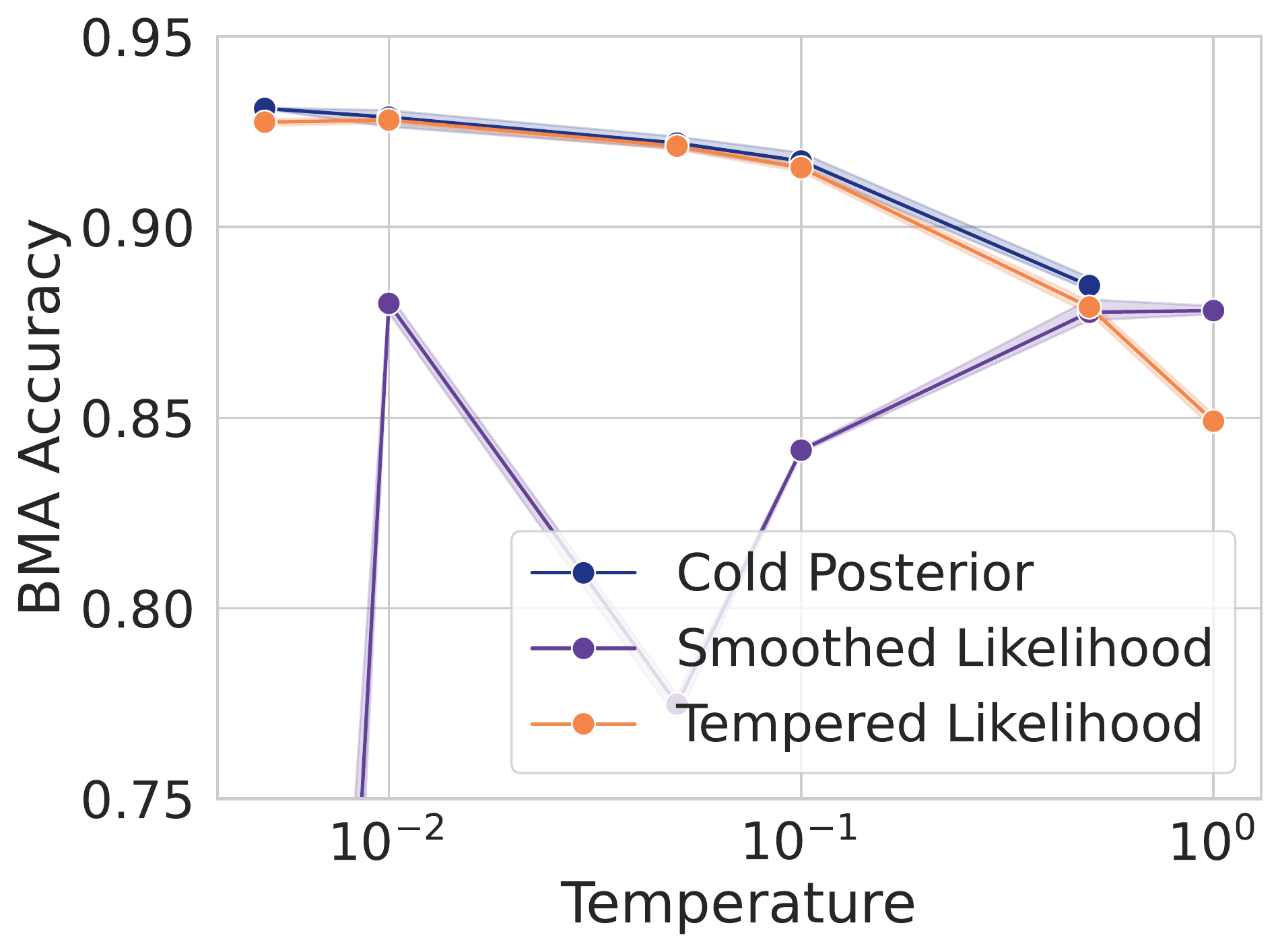}
    \caption{
    ResNet-$18$ on CIFAR-$10$. Cold posteriors are matched by tempered posteriors, but not by the smoothed likelihood.
    }
    \label{fig:smooth_vs_tempered_softmax}
    \vspace{-2em}
\end{wrapfigure}

\cref{thm:rescaled_priors} shows that the cold posterior effect in the Bayesian deep learning literature can be equivalently thought of in terms of a \emph{tempered likelihood} posterior alone. The only difference is that the actual priors being used are different from the stated priors in that they have tighter variances. See \cref{sec:conn_cold_tempered} for further examples and discussion.

The above result unifies several tempered likelihood findings \citep{zhang2019cyclical,heek2019bayesian} to those with the cold posterior studies \citep{wenzel2020good,fortuin2021bayesian}. Using a ResNet-$18$ on CIFAR-$10$, we empirically demonstrate such an equivalence in \cref{fig:smooth_vs_tempered_softmax}. A rescaled prior and step size $\epsilon$ can match the performance of a cold posterior. To match the gradient scaling, we must downscale the step size by the square root of the implied temperature, while we match the prior variance by setting it to $\epsilon T$. 
We give specific examples of these priors in the next section.

Thus, to understand the cold posterior effect, we may focus on the tempered likelihood term in \cref{eq:tempered_posterior}. Tempering the likelihood with $T < 1$ is sharpening the density function to facilitate a more accurate observation model per our beliefs. It is then only natural to use a likelihood which allows direct control of the sharpness, which we introduce next.
We also note that the likelihood and the prior tend to interact here. 
For a fixed prior, a sharper likelihood will tend to produce a sharper posterior, while decreasing the variance on the prior will additionally tend to produce a sharper (more confident) posterior.
This is exactly the effect produced by likelihood tempering with $T < 1$; tempering the likelihood produces a sharper likelihood which therefore produces a sharper posterior.

\subsubsection{More Cold Priors}
\label{app:cold_priors}

\cref{thm:rescaled_priors} covers many common prior choices in the Bayesian deep learning literature.
For \textit{Gaussian and Laplace priors}, the scaling term (variance in the case of the Gaussian prior and scale in the case of the Laplace prior) changes by a factor of $\epsilon T$ where $\epsilon$ is the original scale.
It also holds for correlated Gaussian priors such as those studied by \citet{fortuin2021bayesian,izmailov2021dangers}.
This effect was previously noted for Gaussian priors by \citet{aitchison2020statistical}.

For \textit{Student's $t$ priors}, we show below that the degrees of freedom increases and the scale term changes producing a distribution with smaller variance than the original distribution.
Similarly for \textit{Cauchy priors}, the entire form of the distribution changes and in such a way that it now has finite mean and variance, which it did not before.

The log pdf for a non-central student T distribution\footnote{\url{https://en.wikipedia.org/wiki/Noncentral_t-distribution}} is
\begin{align*}
    \log p(x) = \log \Gamma(v + 1/2) - \log \Gamma (v/2) - \log \sqrt{v \pi} \sigma - \frac{v + 1}{2}\log \left(1 + \frac{x^2}{v \sigma^2}\right)
\end{align*}
Rescaling by a temperature $T$ gives
\begin{align*}
    \frac{1}{T}\log p(x) = \frac{1}{T} \left(\log \Gamma(v + 1/2) - \log \Gamma (v/2) - \log \sqrt{v \pi} \sigma\right) - \frac{v + 1}{2T}\log \left(1 + \frac{x^2}{v \sigma^2}\right)
\end{align*}
and letting $\tilde v = \frac{v + 1}{T} - 1$ produces
\begin{align*}
    \frac{1}{T}\log p(x) &= \text{const.} - \frac{\tilde v + 1}{2}\log \left(1 + \frac{\tilde v}{\tilde v}\frac{x^2}{v \sigma^2}\right) 
    =\text{const.} - \frac{\tilde v + 1}{2}\log \left(1 + \frac{x^2}{\tilde v \tilde \sigma^2}\right)
\end{align*}
with $\tilde \sigma^2 = v \sigma^2 / \tilde v = \frac{v \sigma^2 T}{v + 1 - T}.$

For $v > 2,$ the original variance was $\sigma^2 \frac{v}{v - 2}$ and now it is $\sigma^2 \frac{T}{v + 1 - 3T},$ which is less whenever $v > 2$ and so the variance exists.
Note additionally that $\tilde v > v$ as well.

For a Cauchy distribution (or equivalently, follow the argument for the non-central $t$ with $\eta = 1$), the tempered prior follows the form 
\begin{align*}
    p_T(x) \propto (1 + x^2)^{-1/T},
\end{align*}
which produces a finite integral and therefore a proper prior for all $T \geq 1.$
This class of distributions has finite means for any $T < 1$ and finite variances for $T < 2/3$, whereas the Cauchy distributions do not have finite means or variances.
This change may have the effect of making posterior means exist for cold posteriors when they may not have existed before, as can be the case for logistic regression with Cauchy priors \citep{ghosh2018use}.

\subsection{Discussion with the Cold Posterior Effect}

These types of findings tend to suggest that ``cold" priors have smaller variances (or spreads, if the variances do not exist) than their warm counterparts. 
This is most clearly the case for Gaussian and Laplace prior distributions where the prior variance is reduced from $\epsilon$ to $\epsilon T$.
One natural question is if this finding is an explanation for the ``cold" posterior effect, an analogue of prior-misspecification, as discussed by \citet{wenzel2020good}. However, \citet{wilson2020bayesian} found that standard Gaussian $\mathcal{N}(0, 1)$ priors tend to be reasonably well-calibrated and that varying priors tends to have minimal impact on down-stream accuracy and calibration.
\citet{izmailov2021bayesian} verified this evidence with high-quality HMC samples, but without data augmentation.

\begin{figure}[!ht]
\centering
\includegraphics[width=.7\linewidth]{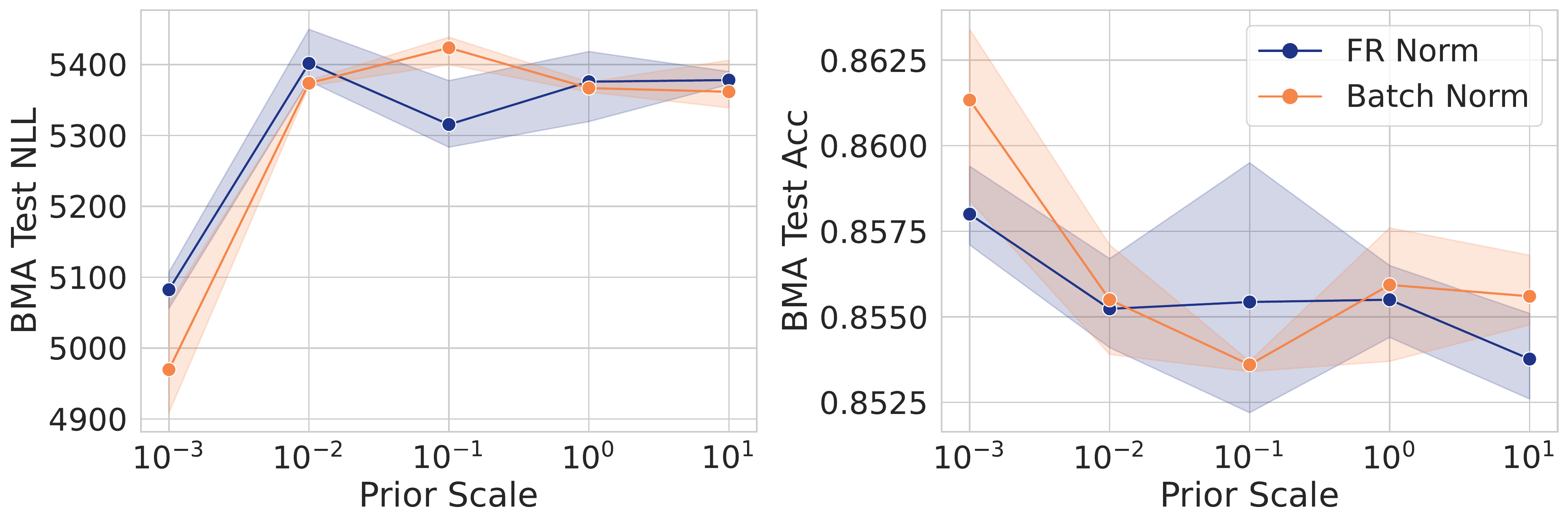}
\caption{Priors of various scales for BNNs with temperature $1$. No prior with the same step size is able to match the performance of the lower temperature.}
\label{fig:prior_scales}
\end{figure}

In \cref{fig:prior_scales}, we show that no choice of prior alone with temperature $1$ is able to reach the same accuracy as a tempered (or cold) baseline model.
Thus, it is not merely (Gaussian) prior misspecification that causes cold posteriors to exist, and we need to look beyond this potential explanation.

\section{Visualizations of Tempered Cross-Entropy and Noisy Dirichlet Likelihoods}
\label{sec:app_likelihood_viz}

\begin{figure}[!ht]
\centering
\includegraphics[width=0.7\linewidth]{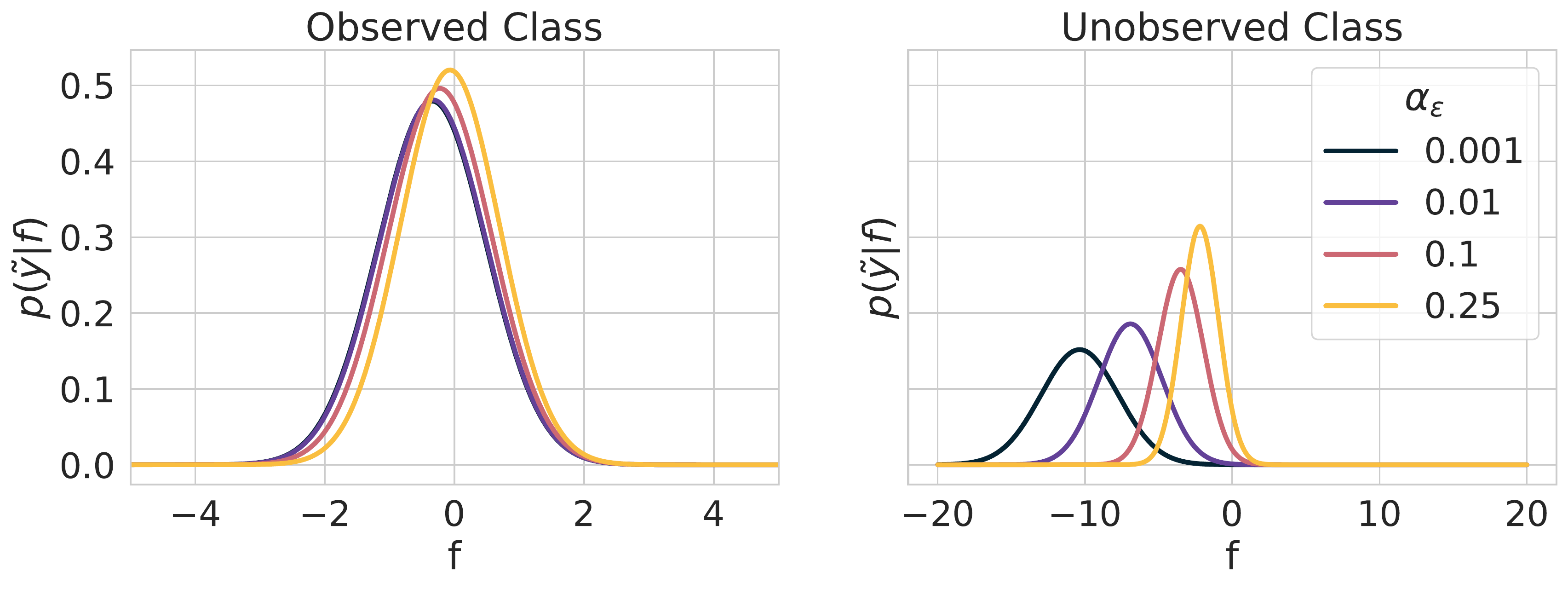}
\caption{Dirichlet densities for various values of $\alpha_\epsilon$ for both the observed class (\textbf{left}) and any unobserved classes (\textbf{right}). Higher values of $\alpha_\epsilon$ tend to produce sharper likelihoods as the distributions corresponding to these values have lower variances.
Here $f$ is the equivalent to the logit, while $\tilde{y}$ is the rescaled observation.
}
\label{fig:noisy_dirichlet}
\end{figure}

\section{Targeting Distributions for Variational Inference}\label{app:alt_inferences}

A naive attempt to incorporate data augmentation would be to simply stack all $K$ copies of the data into the posterior and then to infer it; however, this approach would contract the posterior by too much as we would be considering $NK$ data points rather than $N$ data points.
A second approach would be then to down-weight the augmentations so that we effectively see only $N$ data points by raising the likelihood of each (augmented) data point to a power $1/K$, which produces exactly \cref{eq:implied_posterior}.

Variational inference with stochastic gradient descent using the evidence lower bound (ELBO) produces a similar gradient to \cref{eq:stochastic_gradient} as 
\begin{align}
    \nabla \text{ELBO}(\phi) =&~ \frac{N}{m} \sum_{(x_i,y_i) \in \dset_m} \mathbb{E}_{q_\phi(\params)} \nabla_\params \log{p(y_i \mid t_j(x_i))}
+ \nabla_\phi \text{KL}(q_\phi(\params) || p(\params)), \label{eq:aug_vi_grad}
\end{align}
switching the parameters of the variational distribution to $\phi$ and again with transformations $t_j$ sampled uniformly from $\mathcal{T}.$
The same two sources of randomness, sub-sampling and augmentation, apply in this situation and computing the expected value of \cref{eq:aug_vi_grad} produces a similar target as \cref{eq:aug_sgld_grad_expectation}:
\begin{align}
        \mathbb{E}[\nabla \text{ELBO}(\phi)] =&~ \sum_{i=1}^N \sum_{j=1}^K [\nabla_\params \log{p(y_i \mid t_j(x_i))^{1/K}}]
+ \nabla_\phi \text{KL}(q_\phi(\params) || p(\params)),
\label{eq:aug_vi_grad_expectation}
\end{align}
implying that the full optimization problem VI is targeting uses a tempered likelihood in that term, e.g. $p(y_i \mid t_j(x_i))^{1/K}$.
We note that variational inference for BNNs often does end up requiring tempering by downweighting the KL term in the ELBO, see Section 2.3 of \citet{wenzel2020good} for several examples.
While we do not consider Laplace approximations under data augmentation, they additionally tend to require some amount of likelihood down-weighting for high performance (see e.g. \citet{immer2021scalable} for example), suggesting that they also tend to target a similar likelihood distribution.
\section{A Proper Likelihood for Data Augmentation}
\label{sec:app_aug_lik}

Treating data augmentations as independent samples from the true underlying data generating process is problematic, since the augmented samples are highly correlated with the originally observed samples. In \cref{sec:data_aug_gp}, we see that stacking the augmented samples in to a single unified dataset is problematic in the sense that observation noise artificially inflated. For well-calibrated predictions, this is a mischaracterization of the data generating process.

To arrive at a proper likelihood accounting for data augmentations $t_j \in \mathcal{T}$, consider a likelihood which extends the usual observation likelihood with a consistency term,
\begin{equation}
    \label{eq:aug_likelihood}
    p_{\mathrm{aug}}(\dset \mid \params) = \prod_{i=1}^N p(y_i \mid x_i) \cdot \prod_{t_j \in \mathcal{T}} p\left(f(t_j(x_i); \params)) \mid f(x_i; \params) \right).
\end{equation}
The augmentation likelihood $p_{\mathrm{aug}}$ properly accounts for aleatoric uncertainty when the augmentation is trivial, and not lead to underconfident predictions when the augmentations are uncorrelated. For a minibatch sample from the full dataset $\dset_m \subset \dset$ of size $m$, and a subset of $k$ augmentations $\mathcal{T}_k \subset \mathcal{T}$, the valid unbiased stochastic gradient estimator for use with SGLD is given by,
\begin{align}
\begin{split}
\nabla \widetilde{U}_{\mathrm{aug}}(\params) =&~ \frac{N}{m} \sum_{(x_i,y_i) \in \dset_m} \nabla_\params \log{p(y_i \mid t_k(x_i))} + \nabla_\params \log{p(\params)} \\
&~+ \cdot \frac{N}{m} \sum_{x_i \in \dset_m} \frac{K}{k} \sum_{t_j \in \mathcal{T}_k} \nabla_\params \log p\left(f(t_j(x_i); \params)) \mid f(x_i; \params) \right)  .
\end{split}\label{eq:stochastic_gradient_aug}
\end{align}
Note that to use the augmentation likelihood in \cref{eq:aug_likelihood}, we need to know the number $K$ of possible augmentations that we consider.

\section{Further Experimental Results}\label{app:further_experiments}
In the top left panel of \cref{fig:noisy_dirichlet_perf}, we demonstrate that softmax scaling in the presence of label noise produces very underconfident predictions.

For GPs, varying the values of $\alpha_\epsilon$ then vary the sharpness of the resulting posterior predictive distribution, as we show in \cref{fig:noisy_dirichlet_perf} for a two spirals problem with $150$ data points and $20$ noisy labels.
The posterior predictive for $\alpha_\epsilon = 0.001$ is only confident in a very small region of the data, whereas the higher values of $\eta$ are confident in a broader region indicating a sharper likelihood that enables more certainly about the observed data.
This is in comparison to a Laplace approximated Bernoulli GP classifier that is less confident about the data than $\alpha_\epsilon = 0.1$ but performs similarly in terms of confidence to the Dirichlet with $\alpha_\epsilon = 0.25,$ suggesting that too sharp of distributions can reduce some confidence.

\begin{figure}[!ht]
\centering
\includegraphics[width=\linewidth]{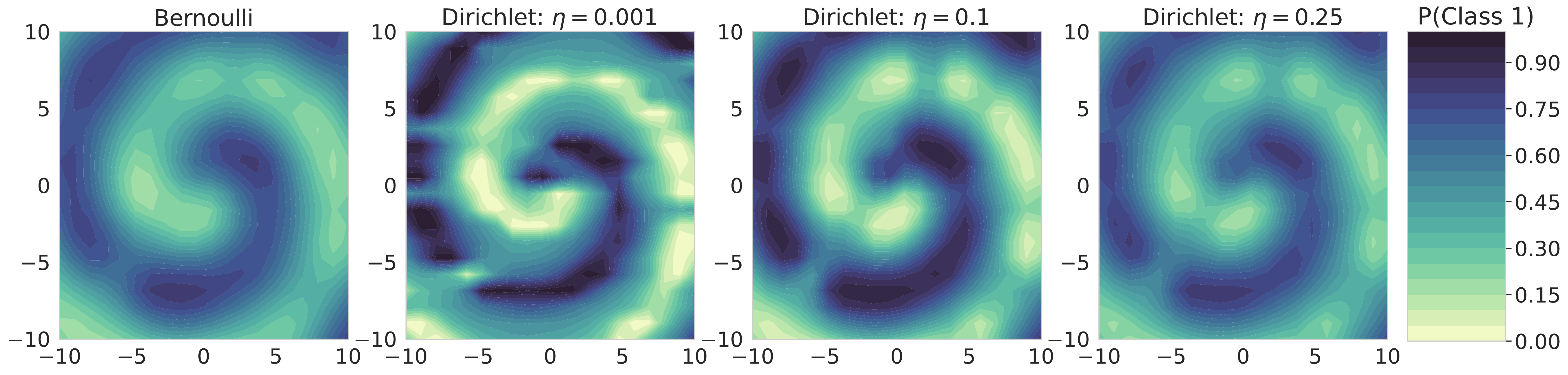}
\caption{Modeling label noise explicitly using \cref{eq:noisy_dirichlet_gauss} and using tuned values of $\eta$ enables better modelling of aleatoric uncertainty in the presence of label noise for GP classifiers on a two spirals problem with $150$ data points and $20$ flipped labels.}
\label{fig:noisy_dirichlet_perf}
\end{figure}

\begin{figure}[!ht]
\centering
\includegraphics[width=0.8\linewidth]{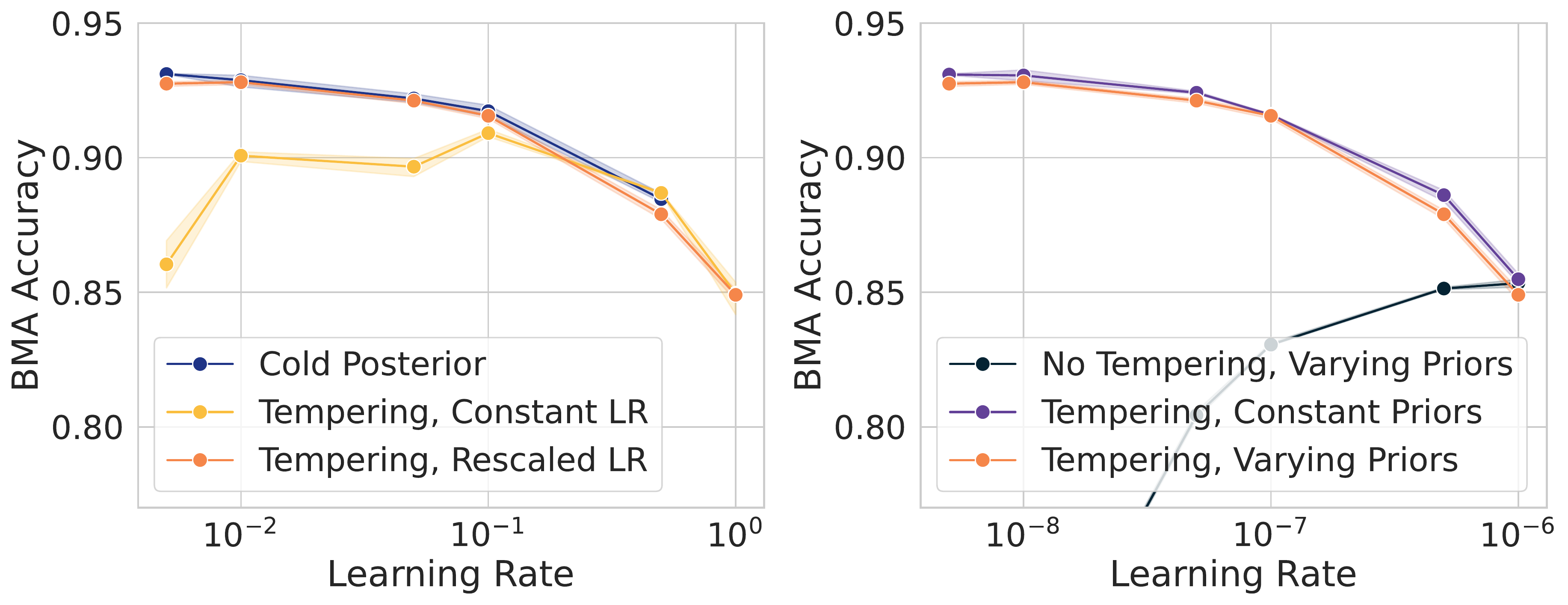}
\caption{We empirically verify that a tempered likelihood combined with prior rescaling is sufficient to reproduce the cold posterior effect, as discussed in \cref{sec:conn_cold_tempered}.}
\end{figure}
\section{Experimental Details}\label{app:experimental_details}

We use PyTorch for all experiments \citep{paszke2019pytorch}.
For the Gaussian process experiments, we used GPyTorch \citep{gardner2018gpytorch} for all but the Laplace bernoulli classifier where we used scikit-learn \citep{pedregosa2011scikit}.

We used $\mathcal{N}(0,1)$ priors unless otherwise stated and ran cyclical learning rate SGLD for $1000$ epochs with an initial learning rate of $10^{-6}$ and a momentum term of $0.99$.
We used ResNet-$18$ \citep{he2016identity} architectures unless otherwise stated and used the CIFAR-$10$\footnote{\url{https://www.cs.toronto.edu/~kriz/cifar.html} dataset}.

For experiment management, we used wandb \citep{wandb} to manage experiments.
Experiments were performed on a mix of Nvidia RTX GPUs on internal servers and clusters as well as some RTX GPUs on AWS; we also used some AMD mi50 GPUs on an internal cluster.
Each SGLD trial took roughly $10$ hours to run, with a total of 3 trials per experiment.
Unless otherwise mentioned, all experiments used $3$ random seeds and we plotted the mean with shading via one standard deviation.

\subsection{Synthetic problem}
\label{sec:app_exp_synthetic}

We present the code used to generate the data for the synthetic problem considered in \cref{sec:exp_synthetic} in \cref{code:synthdata}.

For all models we use an iid Gaussian prior $\mathcal N(0, 0.3^2)$ over the parameters.
We use HMC with a step size $3 \cdot 10^{-6}$ and the trajectory length is $\pi \cdot 0.3 / 2$, amounting to $150\cdot 10^3$ leapfrog steps per iteration, following the advice in \citet{izmailov2021bayesian}.
We run HMC for 100 iterations, discarding the first 10 samples as burn-in.

For data augmentation, we generate $K = 4$ augmentations of each datapoint.
We run HMC on the distribution in \cref{eq:implied_posterior}.
For the tempered likelihood, we use $T=0.1$, and for the noisy Dirichlet model we use $\alpha_\epsilon=10^{-5}$.

\begin{lstlisting}[language=Python, caption=Data generation for the synthetic problem, label={code:synthdata}]
import numpy as np
    
def twospirals(n_samples, noise=.5, random_state=920):
    """
     Returns the two spirals dataset.
    """
    onp.random.seed(random_state)
    
    n = np.sqrt(np.random.rand(n_samples,1)) * 600 * (2*np.pi)/360
    d1x = -1.5*np.cos(n)*n + np.random.randn(n_samples,1) * noise
    d1y =  1.5*np.sin(n)*n + np.random.randn(n_samples,1) * noise
    return (np.vstack((np.hstack((d1x,d1y)), np.hstack((-d1x,-d1y)))),
            np.hstack((np.zeros(n_samples), np.ones(n_samples))))

x, y = twospirals(n_samples=200, noise=0.6, random_state=920)

label_mask = (np.linalg.norm(x, axis=-1) > 13)
y[label_mask] = 1 - y[label_mask]

mask = np.logical_and(x[:, 0] < 0, x[:, 1] < 0)
x, y = x[mask], y[mask]
x_aug = np.concatenate([x, -x, x * onp.array([[1., -1]]), x * onp.array([[-1., 1]])])
y_aug = np.concatenate([y, y, y, y])
\end{lstlisting}

\end{document}